%% file: cmistep_neurips.tex
\documentclass{article}

  \usepackage[final,nonatbib]{neurips_2020}

\usepackage[utf8]{inputenc} %
\usepackage[T1]{fontenc}    %
\usepackage{hyperref}       %
\usepackage{url}            %
\usepackage{booktabs}       %
\usepackage{amsfonts}       %
\usepackage{nicefrac}       %
\usepackage{microtype}      %

\usepackage[T1]{fontenc}
\usepackage[utf8]{inputenc}
\usepackage{mathscinet}
\input{headers.tex}

\input{defs.tex}

\newcommand{\ssindex}[1]{U^{(#1)}}

\newcommand{\EGE}{\ensuremath{\mathrm{EGE}_{\Dist}(\Alg)}}
\newcommand{\IWS}{\ensuremath{\mathrm{IOMI}_{\Dist}(\Alg)}\xspace}
\newcommand{\SZB}{\ensuremath{\mathrm{CMI}^{\,k}_{\Dist}(\Alg)}\xspace}

\title{Sharpened Generalization Bounds based on Conditional Mutual Information and an Application to Noisy, Iterative Algorithms}

\input{authors.tex}

\begin{document}

\maketitle

\begin{abstract}

\input{abstract.tex}
\end{abstract}

\input{body-1.tex}

\input{acknowledgments.tex}

\input{impact-statement.tex}
\printbibliography

\newpage 

\appendix

\input{appendices-1.tex}
\end{document}

%% file: headers.tex
\pdfoutput=1
\usepackage{amsmath,amsthm,wrapfig}
	\usepackage[utf8]{inputenc} 
\usepackage{amsmath, amssymb,bm, cases, mathtools, thmtools}
\usepackage{verbatim}
\usepackage{graphicx}\graphicspath{{figures/}}
\usepackage{multicol}
\usepackage{tabularx}
\usepackage[usenames,dvipsnames]{xcolor}
\usepackage{mathrsfs} 
\usepackage{caption}
\usepackage{algorithm}
\usepackage{algorithmicx}
\usepackage[noend]{algpseudocode}
\usepackage{array,booktabs,arydshln,xcolor}

\usepackage[%
		minnames=1,maxnames=99,maxcitenames=2,
		style=numeric,%
		sortcites=true,
		doi=false,url=false,
		giveninits=true,
		hyperref,natbib,backend=biber]{biblatex}
\renewbibmacro{in:}{%
	\ifentrytype{article}{}{\printtext{\bibstring{in}\intitlepunct}}}

\usepackage[T1]{fontenc}
\usepackage{inconsolata}
\usepackage{xmathptmx}
\usepackage{dsfont}

\usepackage{listings} %

\usepackage{enumitem}
\usepackage{booktabs}       %
\usepackage{nicefrac}       %

\usepackage{lineno}
\usepackage{bbm}

\usepackage{caption}
\usepackage{subcaption}

\usepackage{datetime}

\DeclareMathAlphabet\EuRoman{U}{eur}{m}{n}
\SetMathAlphabet\EuRoman{bold}{U}{eur}{b}{n}

\usepackage[capitalize]{cleveref}

\crefname{lemma}{Lemma}{Lemmas}
\crefname{corollary}{Corollary}{Corollaries}
\crefname{theorem}{Theorem}{Theorems}

\makeatletter
\let\reftagform@=\tagform@
\def\tagform@#1{\maketag@@@{\ignorespaces\textcolor{gray}{(#1)}\unskip\@@italiccorr}}
\renewcommand{\eqref}[1]{\textup{\reftagform@{\ref{#1}}}}
\makeatother

\declaretheorem[style=plain,numberwithin=section,name=Theorem]{theorem}
\declaretheorem[style=plain,sibling=theorem,name=Lemma]{lemma}

\declaretheorem[style=plain,sibling=theorem,name=Corollary]{corollary}

\declaretheorem[style=definition,sibling=theorem,name=Definition]{definition}

\declaretheorem[style=remark,qed=$\triangleleft$,sibling=theorem,name=Remark]{remark}
\numberwithin{theorem}{section}

\usepackage{xparse}
\usepackage{xstring}
\usepackage{xspace}

%% file: defs.tex
\def\[#1\]{\begin{align}#1\end{align}}
\def\*[#1\]{\begin{align*}#1\end{align*}}

\newcommand{\defas}{\vcentcolon=}  %
\newcommand{\minf}[1]{I(#1)}

\newcommand{\entr}[1]{H(#1)}

\newcommand{\cPr}[2]{\Pr^{#1}[#2]}

\newcommand{\parspace}{\mathcal{W}}
\newcommand{\Dist}{\mathcal D}
\newcommand{\dataspace}{\mathcal Z}
\newcommand\optparen[1]{\ifthenelse{\equal{#1}{}}{}{(#1)}}
\newcommand{\RiskChar}{R}
\newcommand{\Risk}[2]{\RiskChar_{#1}\optparen{#2}}
\newcommand{\EmpRisk}[2]{\hat \RiskChar_{#1}\optparen{#2}}

\newcommand{\SurEmpRisk}[2]{\tilde \RiskChar_{#1}\optparen{#2}}

\newcommand{\set}[2][]{#1\{#2 #1\}}
\newcommand{\brackets}[2][]{#1[#2 #1]}

\newcommand{\dist}{\ \sim\ }
\newcommand{\unifdist}{\text{Unif}}

\newcommand{\given}{\mid}

\newcommand{\Naturals}{\mathbb{N}}

\newcommand{\Reals}{\mathbb{R}}

\newcommand{\as}{\textrm{a.s.}}

\newcommand{\grad}{\nabla}

\newcommand{\dee}{\mathrm{d}}

\DeclareMathOperator*{\newlim}{\mathrm{lim}\vphantom{\mathrm{infsup}}}

\DeclareMathOperator*{\newinf}{\mathrm{inf}\vphantom{\mathrm{infsup}}}
\DeclareMathOperator*{\newsup}{\mathrm{sup}\vphantom{\mathrm{infsup}}}
\renewcommand{\lim}{\newlim}

\renewcommand{\inf}{\newinf}
\renewcommand{\sup}{\newsup}

\newcommand{\ProbMeasures}[1]{\mathcal{M}_1(#1)}

\renewcommand{\Pr}{\mathbb{P}}
\def\EE{\mathbb{E}}

\newcommand{\defn}[1]{\textit{#1}}

\newcommand{\equaldist}{\overset{d}{=}}

\newcommand{\card}[1]{\lvert #1 \rvert}

\newcommand{\KLname}{\mathrm{KL}}

\newcommand{\KL}[2]{\KLname(#1 \,\|\,#2)}

\newcommand{\Normal}{\mathcal N}

\newcommand{\loss}{\ell}
\newcommand{\trainset}{S}

\newcommand{\e}{\mathrm{e}}
\newcommand{\id}[1]{\mathbb{I}_{#1}}

\newcommand{\Alg}{\mathcal{A}}

\newcommand{\unif}[1]{\text{Unif}(#1)}

\newcommand{\lcrx}[4][{-1}]{
	\IfEq{#1}{-1}{\left #2 {{{{#3}}}} \right #4}{
   	\IfEq{#1}{0}{#2 {{{{#3}}}} #4}{
	\IfEq{#1}{1}{\bigl #2 {{{{#3}}}} \bigr #4}{
	\IfEq{#1}{2}{\Bigl #2 {{{{#3}}}} \Bigr #4}{
	\IfEq{#1}{3}{\biggl #2 {{{{#3}}}} \biggr #4}{
	\IfEq{#1}{4}{\Biggl #2 {{{{#3}}}} \Biggr #4}{
    \GenericWarning{"4th argument to lcrx must be -1, 0, 1, 2, 3, or 4"}
    }}}}}}}
\newcommand{\rbra}[2][{-1}]{\lcrx[#1] ( {#2} ) }
\newcommand{\sbra}[2][{-1}]{\lcrx[#1] [ {#2} ] }

\newcommand{\rnderiv}[2]{\frac{\text{d} #1}{\text{d} #2}}

\newcommand{\cEE}[1]{\EE^{#1}}
\newcommand{\ww}{W}
\newcommand{\bgrad}[2]{\nabla \SurEmpRisk{#1}{#2}}

\newcommand{\ssgrad}{\nabla \sloss}

\newcommand{\conditional}[1]{#1_{t|}}

\newcommand\eqdist{\mathrel{\overset{\smash{\makebox[0pt]{\mbox{\normalfont\tiny d}}}}{=}}}
\newcommand{\indep}{\mathrel{\perp\mkern-9mu\perp}}

\newcommand{\sloss}{\tilde{\ell}}

\newcommand{\dminf}[2]{I^{#1} (#2)}

\newcommand{\indic}[1]{ \mathds{1} \{ #1\}}

\newcommand{\hypt}[1]{\mathcal{H}_{#1}}

\newcommand{\supersam}[1]{\tilde{Z}^{(#1)}}

\newcommand{\incoh}{\zeta}

\newcommand{\range}[1]{ [#1] }

%% file: authors.tex
\author{
	Mahdi Haghifam\\
	University of Toronto, \\
	Vector Institute \\
	\And
	Jeffrey Negrea\\
	University of Toronto,\\
	Vector Institute \\
	\And
	Ashish Khisti \\
	University of Toronto \\
	\And
	Daniel M. Roy \\
	University of Toronto,\\
	Vector Institute\\
	\And
	Gintare Karolina Dziugaite \\
	Element AI
}

%% file: abstract.tex
The information-theoretic framework of \citeauthor{RussoZou16} (\citeyear{RussoZou16}) 
and \citeauthor{XuRaginsky2017} (\citeyear{XuRaginsky2017})
provides bounds on the generalization error of a learning algorithm in terms of the mutual information between the algorithm's output and the 
training sample. 
In this work, we study the proposal by \citeauthor{steinke2020reasoning} (\citeyear{steinke2020reasoning})
to reason about the generalization error of a learning algorithm 
by introducing a super sample that contains the training sample as a random subset
 and computing mutual information conditional on the super sample.
We first show that these new bounds based on the conditional mutual information are tighter than
those based on the unconditional mutual information.
We then introduce yet tighter bounds, building on the ``individual sample'' idea of 
\citeauthor{BuZouVeeravalli2019} (\citeyear{BuZouVeeravalli2019})
and the ``data dependent'' ideas of \citeauthor{negrea2019information} (\citeyear{negrea2019information}), 
using disintegrated mutual information.
Finally, we apply these bounds to the study of the Langevin dynamics algorithm, 
showing that conditioning on the super sample allows us to exploit information in the optimization trajectory
to obtain tighter bounds based on hypothesis tests.

%% file: body-1.tex
\section{Introduction}

Let $\Dist$ be an unknown distribution on a space $\dataspace$,
and let $\parspace$ be a set of parameters that index a set of predictors, $\loss: \dataspace \times \parspace \to [0,1]$ be a bounded loss function. %
Consider a (randomized) learning algorithm $\Alg$ that selects an element $W$ in $\parspace$,
based on an IID sample
$
\smash{S = \left(Z_1,\hdots,Z_n\right)\sim \Dist^{\otimes n}}.
$
For $w \in \parspace$, let $\Risk{\Dist}{w} = \EE \loss(Z,w)$ denote the risk of predictor $w$,
and  
$\smash{\EmpRisk{S}{w} = \frac{1}{n} \sum_{i=1}^{m} \loss(Z_i,w)}$ denote the empirical risk.
Our interest in this paper is  the \defn{(expected) generalization error} of $\Alg$ with respect to $\Dist$,
\*[
\EGE = \EE \sbra[0]{\Risk{\Dist}{W} - \EmpRisk{S}{W}}.
\]

In this work, we study bounds on generalization error in terms of information-theoretic measures of dependence between the data and the output of the learning algorithm.
This approach was initiated by \citet{RussoZou15,RussoZou16} and has since been extended \citep{raginsky2016information,XuRaginsky2017,asadi2019chaining2,IbrahimEspositoGastpar19,asadi2018chaining,BuZouVeeravalli2019}.
The basic result in this line of work is that the generalization error can be bounded in terms of 
the mutual information $\minf{W;S}$
between the data and the learned parameter, a quantity that has been called the \emph{information usage} or 
\emph{input--output mutual information of $\Alg$ with respect to $\Dist$}, which we denote by $\IWS$.
The following result is due to \citet{RussoZou16} and \citet{XuRaginsky2017}.
\begin{theorem}
\label{thm:RZ_XR_bounds}
$\displaystyle
	\EGE \leq \sqrt{\frac{\IWS}{2n}}.
$
\end{theorem}
Theorem \ref{thm:RZ_XR_bounds} formalizes the intuition that a learning algorithm without heavy dependence on the training set will generalize well. 
This result has been extended in many directions:
\citet{raginsky2016information} connect variants of \IWS to different notions of stability. 
\citet{asadi2018chaining} establish refined bounds using chaining techniques for
subgaussian processes. 
\citet{BuZouVeeravalli2019} obtain a tighter bound by replacing \IWS with the mutual information between $W$ and a single training data point.
\citet{negrea2019information} propose variants that allow for data-dependent estimates.    
See also \citep{JiaoHanWeissman17, LopezJog2018, bassily2018learners,asadi2019chaining2,IbrahimEspositoGastpar19}.

Our focus in this paper is on a new class of information-theoretic bounds on generalization error,
proposed by \citet{steinke2020reasoning}.
Fix $k \ge 2$, let $\smash{[k]=\{1,\dots,k\}}$,
let $\smash{\ssindex{k}=\left(U_1,\hdots,U_n\right)\dist\unifdist([k]^n)}$,
and let 
$
   \smash{\supersam{k} \dist \Dist^{\otimes(k\times n)}}
$
be a $k\times n$ array of IID random elements in $\dataspace$, independent from $\smash{\ssindex{k}}$.
Let
$\smash{
S = \rbra{Z_{U_1,1},\hdots,Z_{U_n,n}}
}$
and let $W$ be a random element in $\parspace$ 
such that
conditional on $S$, $\smash{\ssindex{k}}$,
and $\smash{\supersam{k}}$, $W$ has distribution $\Alg(S)$. 
It follows that, conditional on $S$, $W$ is independent from $\smash{\ssindex{k}}$ and $\smash{\supersam{k}}$.
By construction, the data set $S$ is hidden inside the super sample; the indices $\smash{\ssindex{k}}$ 
specify where. 
\citet{steinke2020reasoning} use these additional structures to define:
\begin{definition}
\label{def:cmi}
The \defn{conditional mutual information of $\Alg$ w.r.t. $\Dist$} is %
$
\smash{\SZB = \minf{W;\ssindex{k}\vert \supersam{k}}.} %
$ 
\end{definition}
Intuitively, $\smash{\SZB}$ captures how well we can recognize which samples from the given super-sample $\smash{\supersam{k}}$ were in the training set, given the learned parameters.
This intuition and the connection of \SZB with the membership attack \citep{shokri2017membership} can be formalized using Fano's inequality, showing that \SZB can be used to lower bound the error of any estimator of $\smash{\ssindex{k}}$ given $W$ and $\smash{\supersam{k}}$. (See \cref{app:fano}.)
\citet{steinke2020reasoning} connect \SZB with well-known notions in learning theory such as distributional stability, differential privacy, and VC dimension,
and establish the following bound \citep[][Thm.~5.1]{steinke2020reasoning} in the case $k=2$, the extension to $k\geq 2$ being straightforward:

\begin{theorem}
\label{thm:cmi_main}
$
	\EGE \leq \sqrt{ \frac{2\SZB}{n}} .
$	
\end{theorem}
This paper improves our understanding of the framework introduced by \citet{steinke2020reasoning},
identifies tighter bounds, and applies these techniques to the analysis of a real algorithm.
In \cref{sec:mi-gen}, we present several formal connections between the two aforementioned information-theoretic approaches for studying generalization. 
Our first result bridges $\smash{\IWS}$ and $\smash{\SZB}$, showing that for any learning algorithm, any data distribution, and any $k$, $\smash{\SZB}$ is less that $\smash{\IWS}$. 
We also show that $\smash{\SZB}$ converges to $\smash{\IWS}$ as $k\to \infty$ when $\card{\parspace}$ is finite. 
In \cref{sec:gen-bounds}, we establish two novel bounds on generalization error using the random index and super sample structure of \citeauthor{steinke2020reasoning}, 
and show that both our bounds are tighter than those based on \SZB. 
Finally, in  \cref{sec:ld-gen-bound}, we show how to construct generalization error bounds for noisy, iterative algorithms using the generalization bound proposed in \cref{sec:gen-bounds}. Using the Langevin dynamics algorithm as our example, we introduce a new type of prior for iterative algorithms that ``learns'' from the past trajectory, using a form of \emph{hypothesis testing}, in order to not ``pay'' again for information obtained at previous iterations.
Experiments show that our new bound is tighter than \citep{li2019generalization,negrea2019information}, especially in the late stages of training, where the hypothesis test component of the bound \emph{discounts} the contributions of new gradients. Our new bounds are non-vacuous for a great deal more epochs than related work, and do not diverge or exceed $1$ even when severe overfitting occurs.
\subsection{Contributions}
\begin{enumerate}[leftmargin=*,itemsep=-0.05em] 
\item We characterize the connections between the $\smash{\IWS}$ and $\smash{\SZB}$. 
   We show that $\smash{\SZB}$ is always less than the $\smash{\IWS}$ for any data distribution, learning algorithms and $k$. 
   Further, we prove that $\smash{\SZB}$ converges to $\smash{\IWS}$ when $k$ goes to infinity for finite parameter spaces.
\item We provide novel generalization bounds that relate generalization to the mutual information between learned parameters and a random subset of the random indices $U_1,\dots U_n$. 
\item We apply our generalization bounds to the Langevin dynamics algorithm by constructing a specific \emph{generalized prior and posterior}. We employ a generalized prior that learns about the values of the indices $\smash{U}$ from the optimization trajectory. To our knowledge, this is the first generalized prior that learns about the dataset from the iterates of the learning algorithm.
\item We show empirically that our bound on the expected generalization error of Langevin dynamics algorithm is tighter than other existing bounds in the literature.  
\end{enumerate}

\renewcommand{\SS}{\mathcal S}
\newcommand{\TT}{\mathcal T}
\subsection{Definitions from Probability and Information Theory}
Let $\SS,\TT$ be measurable spaces, let $\ProbMeasures{\SS}$ be the space of probability measures on $\SS$, and
define a probability kernel from $\SS$ to $\TT$ to be a measurable map from $\SS$ to $\ProbMeasures{\TT}$. 
For random elements $X$ in $\SS$ and $Y$ in $\TT$,
write $\Pr[X] \in \ProbMeasures{\SS}$ for the distribution of $X$ and 
write $\cPr{Y}{X}$ for (a regular version of) the conditional distribution of $X$ given $Y$, viewed as a $\sigma(Y)$-measurable random element in $\ProbMeasures{\SS}$. 
Recall that $\cPr{Y}{X}$ is a regular version if, for some probability kernel $\kappa$ from $\TT$ to $\SS$, we have $\cPr{Y}{X} = \kappa(Y)$ a.s. .
If $Y$ is $\sigma(X)$-measurable then $Y$ is a function of $X$. If random measure, $P$, is $\sigma(X)$-measurable then the measure $P$ is determined by $X$, but a random element $Y$ with $\cPr{X}{Y}=P$ is not $X$ measurable unless it is degenerate.
If $X$ is a random variable,
write $\EE X$ for the expectation of $X$ and 
write $\cEE{Y}X$ or $\EE[X | Y]$ for (an arbitrary version of) the conditional expectation of $X$ given $Y$, which is $Y$-measurable.  
For a random element $X$ on $\SS$ and a probability kernel $P$ from $\SS$ to $\TT$, the composition $P(X)\defas P\circ X$ is a $\sigma(X)$-measurable random measure of a random element taking values in $\TT$. We occasionally use this notation to refer to a kernel $P$ implicitly by the way it acts on $X$. 

Let $P,\ Q$ be probability measures on a measurable space $\SS$.
For a $P$-integrable or nonnegative measurable function $f$,
let $\smash{P[f] = \int f \dee P}$.
When $Q$ is absolutely continuous with respect to $P$, denoted $Q \ll P$,
we write $\smash{\rnderiv{Q}{P}}$ for 
the Radon--Nikodym derivative of $Q$ with respect to $P$.
We rely on several notions from information theory:
The \defn{KL divergence} \defn{of $Q$ with respect to $P$},
denoted $\KL{Q}{P}$, is $\smash{Q[ \log \rnderiv{Q}{P} ]}$ when $Q \ll P$ and $\infty$ otherwise.
Let $X$, $Y$, and $Z$ be random elements, and let $\otimes$ form product measures.
The \defn{mutual information between $X$ and $Y$} 
is
$
\minf{X;Y} = \KL{\Pr[(X,Y)]} { \Pr[X] \otimes \Pr[Y]}.
$
The \defn{disintegrated mutual information between $X$ and $Y$ given $Z$,} is\footnote{
Letting $\phi$ satisfy $\phi(Z)= \dminf{Z}{X;Y}$ a.s., define $\minf{X,Y\vert Z=z} = \phi(z)$. This notation is necessarily well defined only up to a null set under the marginal distribution of $Z$.
}
\*[
 \dminf{Z}{X;Y} = \KL{\cPr{Z}{(X,Y)}}{\cPr{Z}{X} \otimes \cPr{Z}{Y} } .
\]
The \defn{conditional mutual information} of $X$ and $Y$ given $Z$ is $\smash{\minf{X;Y\vert Z} = \EE \dminf{Z}{X,Y}}$.

\section{Connections between $\smash{\IWS}$ and \smash{\SZB}}
\label{sec:mi-gen}
In this section, we compare approaches for the information-theoretic analysis of generalization error, and we aim to unify the two main information-theoretic
approaches for studying generalization.  In \cref{thm:cmi_lower_iomi,thm:k_inf} we
will show that for any learning algorithm and any data distribution, 
\SZB provides a tighter measure of dependence than \IWS, and that one can recover \IWS--based bounds from \SZB for finite parameter spaces. 

A fundamental difference between \IWS and \SZB is that \SZB is bounded by $n\log k$ 
\citep{steinke2020reasoning}, while
\IWS can be infinite even for learning algorithms that provably generalize \citep{BuZouVeeravalli2019}.
One of the motivations of \citeauthor{steinke2020reasoning} was 
that 
proper empirical risk minimization algorithms over threshold functions on $\Reals$ have large $\smash{\IWS}$ \citep{bassily2018learners}. 
In contrast, some such algorithms have small $\smash{\SZB}$. 
Our first result shows that $\smash{\SZB}$ is never larger than $\smash{\IWS}$.
\begin{theorem}\label{thm:cmi_lower_iomi}
For every $k\geq 2$, $\minf{W;S} = \minf{W;\tilde{Z}^{(k)}} + \minf{W;\ssindex{k}\vert \supersam{k}}$
and
$$\SZB \le \IWS.$$
\end{theorem}
Next, we address the role of the size of the super-sample in CMI. In \citep{steinke2020reasoning}, CMI is defined using a super-sample of
size $2n$ ($k = 2$) only. Our next result demonstrates that $\smash{\SZB}$ agree $\smash{\IWS}$ in the limit as $k \to \infty$ when the parameter space is finite.
\begin{theorem}
\label{thm:k_inf}
If the output of $\Alg$ takes value in a finite set then
$$
\lim_{k \to \infty} \SZB = \IWS.
$$
\end{theorem}
Combining \cref{thm:cmi_main,,thm:k_inf}, we obtain
\[
\label{eq:ege-k-inf}
\EGE \leq \lim_{k \to \infty}\sqrt{\frac{2 \SZB}{n}} = \sqrt{ \frac{2 \IWS}{n}}, 
\]
when the parameter space is finite. Comparing \cref{eq:ege-k-inf} with \cref{thm:RZ_XR_bounds} we observe that \cref{eq:ege-k-inf} is twice as large. In \cref{thm:improved_const}, we present a refined bound based on $\smash{\SZB}$ which asymptotically match \cref{thm:RZ_XR_bounds}. The proofs of the results of this section appear in \cref{apx:proof-connections}.
\section{Sharpened Bounds based on Individual Samples}
\label{sec:gen-bounds}
We now present two novel generalization bounds and show they provide a tighter characterization of the generalization error than \cref{thm:cmi_main}. 
The results are inspired by the improvements on $\smash{\IWS}$ due to \citet{BuZouVeeravalli2019}.
In particular, \cref{thm:gen-bound-rndsub} bounds the expected generalization error
in terms of the mutual information between the output parameter and a random subsequence of the indices $\ssindex{2}$, given the super-sample. \cref{thm:gen-bu-style} provides a generalization bound in terms of the disintegrated mutual information between
each individual element of $\ssindex{2}$ and the output of the learning algorithm, $\ww$. 
The bound in \cref{thm:gen-bu-style} is an analogue of \citep[][Prop.~1]{BuZouVeeravalli2019} for \cref{thm:cmi_main}. 
In this section as in \citet{steinke2020reasoning}, we only consider $\supersam{k}$ and $\ssindex{k}$ with $k=2$, so
we will drop the superscript from $\smash{\ssindex{k}}$.
Let $U = \rbra {U_1,\hdots,U_n }$. The proofs for the results of this section appear in \cref{apx:proof-bounds}.
 
\begin{theorem}
\label{thm:gen-bound-rndsub}
Fix $m \in [n]$ and let $J=\left(J_1,\hdots,J_m\right)$ be a random subset of $\range{n}$, 
distributed uniformly among all subsets of size $m$ and independent from $W,~\tilde{Z}^{(2)}$, and $U$.
Then
\[
	\EGE \leq 	\EE\sqrt{ \frac{ 2  \dminf{\supersam{2}}{W;U_J\vert J}}{m} }.
\]
\end{theorem}
By applying Jensen's inequality to \cref{thm:gen-bound-rndsub}, we obtain 
\[
\label{eq:gen-rnd-sub-after-jensen}
\EGE \leq \sqrt{ \frac{2 \minf{W;U_J\vert \supersam{2},J}}{m} }.
\]
Our next results in \cref{thm:gen-rndsub-prop} let us compare \cref{eq:gen-rnd-sub-after-jensen} for different values of $m=\card{J}$. 
\begin{theorem}
\label{thm:gen-rndsub-prop}
Let $m_1< m_2 \in \range{n}$,
and let $J^{(m_1)},J^{(m_2)}$ be random subsets of $[n]$,
distributed uniformly among all subsets of size $m_1$ and $m_2$, respectively, and 
independent from $W,~\tilde{Z}^{(2)}$, and $U$.
Then
\[
\label{eq:monoton-wrt-m}
\frac{\minf{W;U_{J^{(m_1)}}\vert \supersam{2},J^{(m_1)} }}{m_1} 
\leq \frac{\minf{W;U_{J^{(m_2)}}\vert \supersam{2} ,J^{(m_2)}}}{m_2}.
\]
Consequently, taking $m_2 = n$, for all $1\leq m_1\leq n$
\[
\label{eq:our-better-cmi}
 \EE\sqrt{ \frac{2\ \dminf{\supersam{2}}{W;U_{J^{(m_1)}} \vert J^{(m_1)}}}{m_1} } \leq   \sqrt{ \frac{ 2 \minf{W;U\vert \supersam{2}}}{n}}.
\]
\end{theorem}  
\begin{corollary}
$
	\EGE \leq  \sqrt{2\ \minf{W;U_J\vert \supersam{2} , J} /m}.
$
The case $m=\card{J}=n$ is equivalent to \cref{thm:cmi_main}. 
The bound is increasing in $m \in \range{n}$, and, the tightest bound is achieved when $m=\card{J}=1$. Also, \cref{eq:our-better-cmi} shows our bound in \cref{thm:gen-bound-rndsub} is tighter than \cref{thm:cmi_main} for $k=2$.
\end{corollary}
To further tighten \cref{thm:gen-rndsub-prop} when $m=1$, we show that we can pull the expectation over both $\supersam{2}$ and $J$ outside the concave square-root function.
\begin{theorem}
\label{thm:gen-bu-style}
Let $J\dist\unifdist([n])$ (i.e., $m=1$ above) be
independent from $W,~\tilde{Z}^{(2)}$, and $U$.
Then
\[
\EGE \leq 
&\EE \sqrt{2 \dminf{\supersam{2},J}{W;U_J} }
=  \frac{1}{n}  \sum_{i=1}^{n} \EE \sqrt{2 \dminf{\supersam{2}}{W;U_i} }.
\]
\end{theorem}
\begin{remark}
\cref{thm:gen-bu-style} is tighter than \cref{thm:cmi_main} since
\[
 \frac{1}{n}  \sum_{i=1}^{n} \EE \sqrt{2 \dminf{\supersam{2}}{W;U_i} }
 \leq  \sqrt{ \sum_{i=1}^{n}\frac{2}{n} \minf{W;U_i\vert \supersam{2}} }
 \leq  \sqrt{\frac{2}{n}\minf{W;U\vert \supersam{2}}}
\]
The first inequality is Jensen's,
while the second follows from the independence of indices $U_i$.
\end{remark}
\subsection{Controlling CMI bounds using KL Divergence}
It is often difficult to compute MI directly. One standard approach in the literature is to bound MI by the expectation of the KL divergence of the conditional distribution of the parameters given the data (the ``posterior'') with respect to a ``prior''.  The statement below is adapted from  \citet{negrea2019information}.
\begin{lemma}\label{lemma:variation-mi-kl}
Let $X$, $Y$, and $Z$ be random elements.
For all $\sigma(Z)$-measurable random probability measures $P$ on the space of $Y$,
\*[
	\dminf{Z}{X; Y} 
      &\leq \cEE{Z}[\KL{\cPr{X,Z}{Y}}{P}]\ \textrm{a.s.,}
      && \text{with a.s. equality for } P= \cEE{Z}[\cPr{X,Z}{Y}]= \cPr{Z}{Y}.
\]
\end{lemma}
We refer to the conditional law of $W$ given $S$ as the \defn{``posterior"  of $W$ given $S$}, which we denote $\smash{Q=\cPr{S}{W}=\cPr{\supersam{2},U}{W}}$, and to $P$ as the \defn{prior}.
This can be used in combination with, for example, \cref{lemma:variation-mi-kl} and \cref{thm:cmi_main} to obtain that for any $\smash{\supersam{2}}$-measurable random prior $\smash{P(\supersam{2})}$
\[
\label{eq:gen-cmi-kl}
\EGE
		 \leq \sqrt{ \frac{ 2\ \minf{W;U\vert \supersam{2}}}{n} }\leq \sqrt{ \frac{ 2 \EE [\KL{Q}{P(\supersam{2})}]}{n}}.
\]
Note that the prior only has access to $\smash{\supersam{2}}$, therefore from its perspective the training set can take $2^{n}$ different values. 
Alternatively, combining \cref{lemma:variation-mi-kl} and \cref{thm:gen-bound-rndsub} yields
\[
	\EGE
		 &\leq \EE \sqrt{ \frac{2\ \cEE{\supersam{2}} {\dminf{\supersam{2}}{W;U_J\vert U_{J^c},J} }}{m} }
         \leq \EE \sqrt{ \frac{2 \cEE{\supersam{2}}{[\KL{Q}{P(\supersam{2},U_{J^c},J})]}}{m}} .
\label{eq:gen-rndsub-kl}
\]
In \cref{eq:gen-rndsub-kl} the prior has access to $n-m$ samples in the training set, $S_{J^c}$, because $\smash{\supersam{2}_{U_{J^c}}=S_{J^c}}$. However, since $\supersam{2}$ is known to the prior, the training set can take only $2^m$ distinct values from the point of view of the prior in \cref{eq:gen-rndsub-kl}. 
This is a significant reduction in the amount of information that can be carried by the indexes in $U_J$ about the output hypothesis. Consequently, priors can be designed to better exploit the dependence of the output hypothesis and the index set.

\subsection{Tighter Generalization bound for the case $m=1$}
Since the strategy above controls MI-based expressions via KL divergences, one may ask whether a bound derived with similar tools, but directly in terms of KL, can be tighter than the combination \cref{lemma:variation-mi-kl,thm:gen-bound-rndsub}. The following result shows that for $m=1$ a tighter bound can be derived by pulling the expectation over both $U_{J^c}$ and $J$ outside the concave square-root function.
\begin{theorem}
\label{thm:gen-kl-m1}
Let $J\dist\unifdist(\range{n})$ be independent from $W,~U$, and $\smash{\supersam{2}}$. Let $\smash{Q=\cPr{\supersam{2},U}{W}}$ and $P$ be a $\smash{\sigma \rbra[0]{\supersam{2},U_{J^c},J }}$-measurable random probability measure. 
Then
\[
\label{eq:gen-kl-m1}
	\EGE
		& \leq \EE\sqrt{2\ \KL{Q}{P} } .
\]
\end{theorem}
Here, the KL divergence is between two $\sigma (\supersam{2},J,U)$-measurable random measures, so is random.

\section{Generalization bounds for noisy, iterative algorithms}
\label{sec:ld-gen-bound}
We apply this new class of generalization bounds to non-convex learning. We analyze the Langevin dynamics (LD) algorithm \citep{gelfand1991recursive}, following the analysis pioneered by \citet{PensiaJogLoh2018}. 
The example we set here is a blueprint for building bounds for other iterative algorithms. 
Our approach is similar to the recent advances by \citet{negrea2019information,li2019generalization}, employing data-dependent estimates to obtain easily simulated bounds. We find our new results allow us to exploit past iterates to obtain tighter bounds. The influence of past iterates is seen to take the form of a hypothesis test.

\subsection{Bounding Generalization Error via Hypothesis Testing}
The chain rule for KL divergence is a key ingredient of information-theoretic generalization error bounds for iterative algorithms \citep{PensiaJogLoh2018,negrea2019information,li2019generalization,BuZouVeeravalli2019}. 
$\parspace^{\set{0,\dots,T}}$ denotes the space of parameters generated by an iterative algorithm in $T$ iterations. For any measure, $\nu$, on $\parspace^{\set{0,\dots T}}$, and $W\sim \nu$, let $\nu_0$ denote the marginal law of $W_0$, and $\conditional{\nu}$ denote the conditional law of $W_t$ given $W_0\dots W_{t-1}$.
\begin{lemma}[Chain Rule for KL] \label{lem:kl-decomp}
Let $Q,P$ be probability measures on $\parspace^{\set{0,\dots,T}}$ with $Q_0 = P_0$. The following lemma bounds the KL divergence involving the posterior for the terminal parameter with one involving the sum of the KL divergences over each individual step of the trajectory. 
Then
\*[\textstyle
	\KL{Q_T}{P_T} \leq \KL{Q}{P} = \sum_{t=1}^{T} Q_{0:(t-1)} \brackets{ \KL{\conditional{Q}}{\conditional{P}}}
\]
\end{lemma}
The benefit of using the chain rule to analyze the iterative algorithm are two-fold: first, we gain analytical tractability; many bounds that appear in the literature implicitly require this form of incrementation \citep{li2019generalization,PensiaJogLoh2018,negrea2019information,BuZouVeeravalli2019}. Second, and novel to the present work, the \emph{information in the optimization trajectory can be exploited} to identify $U$ from the history of $W$. 

In order to understand how the prior may take advantage of information from the optimization trajectory, consider applying \cref{lem:kl-decomp} to the KL term in \cref{eq:gen-rndsub-kl}. We have
\*[
\KL{Q_T}{P_T\big(\supersam{2},U_{J^c},J\big)} \leq \sum_{t=1}^{T} \cEE{\supersam{2},U_{J^c},J}{\brackets{ \KL{\conditional{Q}}{\conditional{P}\big(\supersam{2},U_{J^c},J\big)}}}  .
\]
Here $\conditional{P}\big(\supersam{2},U_{J^c},J\big)$ is a $\sigma \rbra[0]{\supersam{2},U_{J^c},J,W_{0:t-1}}$-measurable random probability measure. 
The prior may use $\smash{U_{J^c}}$, $\smash{\supersam{2}}$, and $J$ to reduce the number of possible values that $U$ can take to $\smash{2^{\card{J}}}$. 
Moreover, since $U_J$ is constant during optimization, $\ww_0,\ww_1,\ww_2,\dots \ww_{t-1}$ may leak some information about $U_J$, and the prior can use this information to tighten the bound by choosing a $\conditional{P}$ that achieves small $\KL{\conditional{Q}}{\conditional{P}}$. 
In the special case where the prior can perfectly estimate $U_J$ from $\ww_0,\ww_1,\ww_2,\dots \ww_{t-1}$, we can set $\conditional{P}=\conditional{Q}$ and $\KL{\conditional{Q}}{\conditional{P}}$ will be zero. As will be seen in the next subsection, we can explicitly design a prior that uses the information in the optimization trajectory for the LD algorithm. 

The process by which the prior can learn from the trajectory can be viewed as an \emph{online hypothesis test}, or binary decision problem, where the prior at time $t$ allocates belief between $2^m$ possible explanations, given by the possible values of $U_J$, based on the evidence provided by $\ww_0,\dots \ww_t$. If the prior is able to identify $U_J$ based on the $\ww$s then the bound stops accumulating, even if the gradients taken by subsequent training steps are large. This means that penalties for information obtained later in training are \emph{discounted} based on the information obtained earlier in training.

\subsection{Example: Langevin Dynamics Algorithm for Non-Convex Learning}
We apply these results to obtain generalization bounds for a gradient-based iterative noisy algorithm, the Langevin Dynamics (LD) algorithm.
For classification with continuous parameters, 
the $0$-$1$ loss does not provide useful gradients.
Typically we optimize a surrogate objective, based on a \emph{surrogate} loss, such as cross entropy.
Write $\sloss: \dataspace \times \parspace \to \Reals$ for the surrogate loss
and let 
$\smash{\SurEmpRisk{S}{w} = \frac{1}{n} \sum_{i=1}^{n} \sloss(Z_i,w)}$ be the empirical surrogate risk.
Let $\eta_t$ be the learning rate at time $t$,
$\beta_t$ the inverse temperature at time $t$ 
and let $\epsilon_t$ be sampled i.i.d. from \ $\Normal(0,\id{d})$.
Then the LD algorithm iterates are given by
\[ \label{eq:ldup} \textstyle
\ww_{t+1} = \ww_t - \eta_t \bgrad{S}{\ww_t} + \sqrt{  \frac{2\eta_t}{\beta_t}} \,\varepsilon_t.
\]
\paragraph{The prior}
We will take $m=1$, and construct a bespoke $\smash{\sigma \rbra[0]{\supersam{2},U_{J^c},J }}$-measurable prior for this problem in order to apply \cref{thm:gen-kl-m1}. 
The prior is based on a \emph{decision function}, $\theta:\Reals\to[0,1]$, which at each time $t+1$ takes in a $\sigma(\ww_0\dots\ww_{t})$-measurable \emph{test statistic}, $\Delta Y_{t}$, and returns a \emph{degree of belief} in favor of the hypothesis $U_J=1$ over $U_J=2$. 
The prior predicts an LD step by replacing the unknown (to the prior) contribution to the gradient of the data point at index $J$ with a $\hat\theta_t=\theta(\Delta Y_{t})$-weighted average of the gradients due to each candidate $\set{Z_{1,J},Z_{2,J}}$.The conditional law of the $t$th iterate under the prior is a $\smash{\sigma \rbra[0]{\supersam{2},U_{J^c},J, \ww_0,\dots \ww_{t} }}$-measurable random measure, as required. 
The exact value of the test statistic is $\Delta Y_{t} = Y_{t,2} - Y_{t,1}$, here the $Y_{0,1}=Y_{0,2}=0$ and $Y_{t,u}$ are defined by the formula in \cref{eq:defn-of-Y}. The conditional law of the $t$th iterate under the prior is described by
\[
\ww_{t+1} = 
   \ww_t 
      -  \textstyle{\frac{\eta_t}{n}\rbra{
          \sum_{\substack{i=1 \\ i\neq J}}^{n} \grad\sloss(Z_i,\ww_t)
         + \hat\theta_t\grad\sloss(Z_{1,J},\ww_t) 
         +(1-\hat\theta_t)\grad\sloss(Z_{2,J},\ww_t)}
      + \sqrt{ \frac{2\eta_t}{\beta_t}} \,\varepsilon_t}.
\]
The test statistic chosen is based on the log-likelihood-ratio test statistic for the independent mean $0$ Gaussian random vectors $\smash{(\epsilon_s)_{s=1}^t}$, which is well known to be \emph{uniformly most powerful} for the binary discrimination of means. Natural choices for $\theta$ are symmetric CDFs, since they treat possible values of $U$ symmetrically, and are monotone in the test statistic. 

We define the \emph{two-sample incoherence} at time $t$ by $
\incoh_{t} = \ssgrad(Z_{1,J},W_t) -\ssgrad(Z_{2,J},W_t).$
$\Theta$ denotes the set of measurable $\theta: \Reals \rightarrow [0,1]$. $Y_{0,1}=Y_{0,2}=0$, and for $t\geq 1$, $Y_{t,1}$ and $Y_{t,2}$ are given by (for $u\in\set{1,2}$)
\begin{equation}\label{eq:defn-of-Y}
\begin{aligned}
Y_{t,u} 
    &\triangleq \sum _{i=1}^{t}  \frac{\beta_{i-1}}{4\eta_{i-1}} \| \ww_{i}-\ww_{i-1}+\eta_{i-1}\frac{n-1}{n}\bgrad{S_{J^c}}{\ww_{i-1}}+\frac{\eta_{i-1}}{n}\ssgrad\left(Z_{u,J},\ww_{i-1}\right)  \|^2.\\
\end{aligned}
\end{equation}

\begin{theorem}[Generalization bound for LD algorithm]
\label{thm:gen_bound_ld}
Let $\set{W_t}_{t\in [T]}$ denote the iterates of the LD algorithm. If $\loss(Z,w)$ is $[0,1]$-bounded then
\begin{equation}
\label{eq:gen_ld_final}
\begin{aligned}
\EE\sbra{ \Risk{\Dist}{\ww_T} - \EmpRisk{S}{\ww_T} }
&\leq \frac{1}{n\sqrt{2}} \inf_{\theta \in \Theta}  
\EE \sqrt{\sum_{t=0}^{T-1} \cEE{\supersam{2},U,J}{ \beta_{t} \eta_{t} \|\incoh_{t} \|^2\Big(\indic{U_J=1}-\theta\left(Y_{t,2}-Y_{t,1}\right)\big)^2 }}.
\end{aligned}
\end{equation}
\end{theorem}
\begin{remark}
For $\theta \in \Theta$ with $1-\theta(x) = \theta(-x)$, \cref{eq:gen_ld_final} simplifies to
\begin{equation}
\label{eq:gen_ld_final_simple}
\begin{aligned}
&\EE\sbra{ \Risk{\Dist}{\ww_T} - \EmpRisk{S}{\ww_T} } \leq \frac{1}{n\sqrt{2}}  
\EE \sqrt{\sum_{t=0}^{T-1} \cEE{\supersam{2},U,J}{ \beta_{t} \eta_{t} \|\incoh_{t} \|^2 \theta^2 \big({-1}^{U_J} \left(Y_{t,2}-Y_{t,1}\right)\big)  }}.
\end{aligned}
\end{equation}
For instance $\theta(x)=\frac{1}{2}+\frac{1}{2}\mathrm{tanh}(x)$ and $\theta(x)=\frac{1}{2}+\frac{1}{2}\mathrm{sign}(x)$ satisfy $1-\theta(x) = \theta(-x)$.
\end{remark}
\begin{remark}
By the law of total expectation, for any $\theta \in \Theta$,
$
\EGE 
\leq \frac{1}{2\sqrt{2} n}\EE  \left[V_1 + V_2\right],
$
where
\begin{equation}
 V_u \triangleq \sqrt{ \sum_{t=0}^{T-1} \cEE{\supersam{2},U_{J^c},J,U_J=u}{ \beta_{t} \eta_{t} \|\incoh_{t} \|^2\left(\indic{u=1}-\theta\left( Y_{t,2}-Y_{t,1}\right)\right)^2 }}, \quad u\in \{1,2\}.
\end{equation}
To estimate $V_u$ ($u\in\set{1,2}$) for fixed $J$, the training set is $S_u=\{Z_1,\hdots,Z_{J-1},\tilde Z_{u,J},Z_{J+1},\hdots,Z_{n}\}$. Hence $V_1$, $V_2$ can be simulated from just $n+1$ data points $\smash{\rbra{Z_1,\hdots,Z_{J-1},Z_{J+1},\hdots,Z_{n},\tilde Z_{1,J},\tilde Z_{2,J}} \dist\Dist^{\otimes(n+1)}}$.
\end{remark}
 The generalization bound in \cref{eq:gen_ld_final}  does not place any restrictions on the
learning rate or Lipschitz continuity of the loss or its gradient. In the next corollary we study the asymptotic properties of the bound in \cref{eq:gen_ld_final} when $\sloss$ is $L$-Lipschitz. Then, we
draw a comparison between the bound in this paper and some of the existing bounds in the literature.
\begin{corollary}
\label{cor:gen-lip}
Under the assumption that $\sloss$ is $L$-Lipschitz, we have $\|\incoh_{t}\|\leq 2L$. Then, the generalization bound in \cref{eq:gen_ld_final} can be upper-bounded as
\[
\label{eq:ld-gen-lip}
\EE(\Risk{\Dist}{W_T} - \Risk{S}{W_T}) &\leq 
\frac{\sqrt{2}L}{n} \inf_{\theta \in \Theta}  
\EE \sqrt{\sum_{t=0}^{T-1} \cEE{\supersam{2},U,J}{ \beta_{t} \eta_{t} \left(\indic{U_J=1}-\theta\left(Y_{t,2}-Y_{t,1}\right)\right)^2}}.
\]
\end{corollary}
\begin{remark}
Under an $L$-Lipschitz assumption, for the LD algorithm, \citet[Thm.~9]{li2019generalization} have
\[
\label{eq:iclr2019-bound}
\EE\left[\Risk{\Dist}{W_T} - \Risk{S}{W_T}\right]
    & \leq \textstyle{\frac{\sqrt{2}L}{n}\sqrt{\sum_{t=0}^{T-1} \beta_{t} \eta_{t} }}.
\]
We immediately see that \cref{eq:ld-gen-lip} provides a constant factor improvement over \cref{eq:iclr2019-bound} by na\"ively using $\theta:x\mapsto 1/2$. 
Our bound has order-wise improvement with respect to $n$  over that of \citet{BuZouVeeravalli2019} and \citet{PensiaJogLoh2018} under the $L$-Lipschitz assumption. 
\citet[App.~E.1]{negrea2019information} obtain
\[
\label{eq:lip_bound_neurip}
\EE\left[\Risk{\Dist}{W_T} - \Risk{S}{W_T}\right]
    & \leq \textstyle{\frac{L}{2(n-1)}\sqrt{\sum_{t=0}^{T-1} \beta_{t} \eta_{t} }}.
\]
which is a constant factor better than our bound for the choice $\theta:x\mapsto 1/2$.
However, this $\theta$ essentially corresponds to no hypothesis test, yielding the same prior as in \citep{negrea2019information}.
For more sophisticated choices of decision function ($\theta$), even under a Lipschitz-surrogate loss assumption, it is difficult to compare our bound with related work because the exact impact of $\theta$-discounting is difficult to quantify analytically. 
\end{remark}
\begin{remark}
A prevailing method for analyzing the generalization error in \citep{negrea2019information,BuZouVeeravalli2019,PensiaJogLoh2018,li2019generalization}  for iterative algorithms is via the chain rule for KL, using priors for the joint distribution of weight vectors that are Markov, i.e., given the $t$th weight, the $(t+1)$th weight is conditionally independent from the trajectory so far.
Existing results using this approach accumulate a "penalty" for each step. In \citep{negrea2019information,BuZouVeeravalli2019,li2019generalization},  the penalty terms are, respectively,  the squared Lipschitz constant, the squared norm of the gradients, and the trace of the minibatch gradient covariance. The penalty term in our paper is the squared norm of "two-sample incoherence", defined in \cref{thm:gen_bound_ld} as the squared norm of the difference between the gradient of a randomly selected training point and the held-out point. 
However, the use of chain rule along with existing ``Markovian'' priors introduces a source of looseness, i.e., the accumulating penalty may diverge to $+\infty$ yielding vacuous bounds (as seen in Fig.~1).\emph{ The distinguishing feature of our data-dependent CMI analysis is that the penalty terms get ``filtered'' by the online hypothesis test via our non-Markovian prior, i.e., our prediction for $t+1$ depends on whole trajectory.} When the true index can be inferred from the previous weights, then the penalty essentially stops accumulating.
\end{remark}
\subsubsection{Empirical Results }
\label{subsec:empirical-results}
In order to better understand the effect of discounting and the degree of improvement due to our new bounds and more sophisticated prior, we turn to simulation studies.
We present and compare the empirical evaluations of the generalization bound in \cref{thm:gen_bound_ld} with the data-dependent generalization bounds in \citet{negrea2019information,li2019generalization}. For brevity, many
of the details behind our implementation are deferred to \cref{app:experminet-details}. 
The functional form of our bounds and \citep{negrea2019information,li2019generalization} are nearly identical as all of them use the chain rule for KL divergence. Nevertheless, the summands appearing in the bounds are different. The bound in \citep{li2019generalization} depends on the squared surrogate loss gradients norm, and the generalization bound in \citet{negrea2019information} depends on the squared norm of \textit{training set incoherence} defined as 
$
\|\ssgrad(Z_{J},W_t) - \frac{1}{n-1}\sum_{i \in \range{n}, i \neq J}\ssgrad(Z_{i},W_t)\|^2
$ where the training set is $\{Z_1,\hdots,Z_n\}$ and $J \sim \unif{\range{n}}$. The first key difference between our bound and others is that the summand in our bound consists of two terms: squared norm of the two-sample incoherence, i.e., $\|\incoh_{t}\|^2$, and the squared error probability of a hypothesis test at time $t$, given by the term $\left(\indic{U_J=1}-\theta\left(\sum_{i=0}^{t} \left(Y_{i,2}-Y_{i,1}\right)\right)\right)^2$ in our bound. 
A consequence of this, and the second fundamental difference between our bound and existing bounds, is that our bound exhibits a trade-off in $\|\incoh_{t}\|^2$ because large $\|\incoh_{t}\|^2$ will make the error of the hypothesis test small on future iterations, whereas the bounds in \citep{negrea2019information,li2019generalization} are uniformly increasing with respect to the squared norm of surrogate loss gradients  and the training set incoherence, respectively. In this section we empirically  evaluate and compare our bound with related work across various neural network architectures and datasets.

\begin{figure}
\centering
\begin{subfigure}[b]{.33\linewidth}
  \centering
  \includegraphics[width=1\linewidth, height=3.2cm]{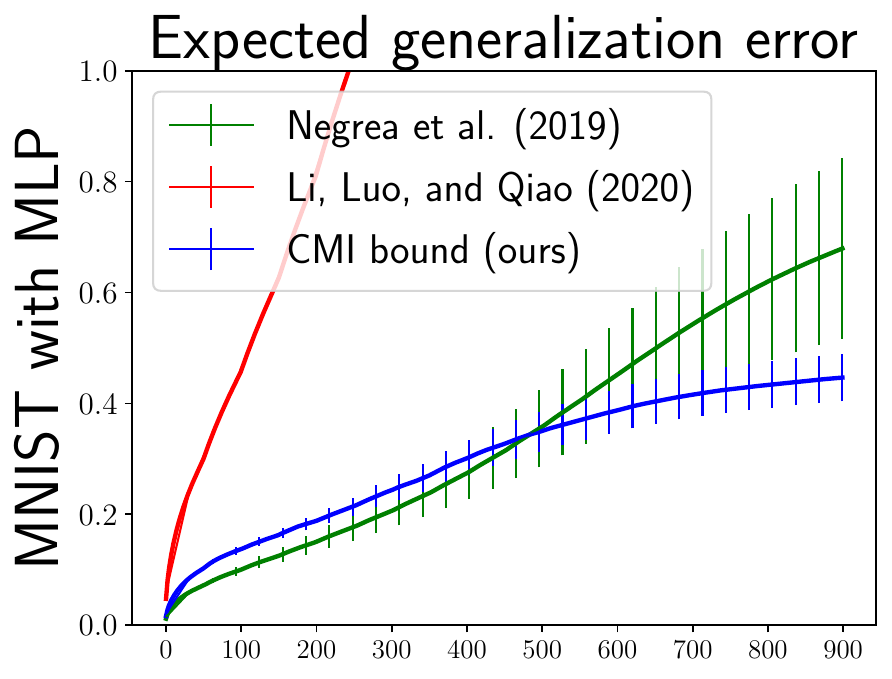}
\end{subfigure}%
\begin{subfigure}[b]{.33\linewidth}
  \centering
  \includegraphics[width=1\linewidth, height=3.2cm]{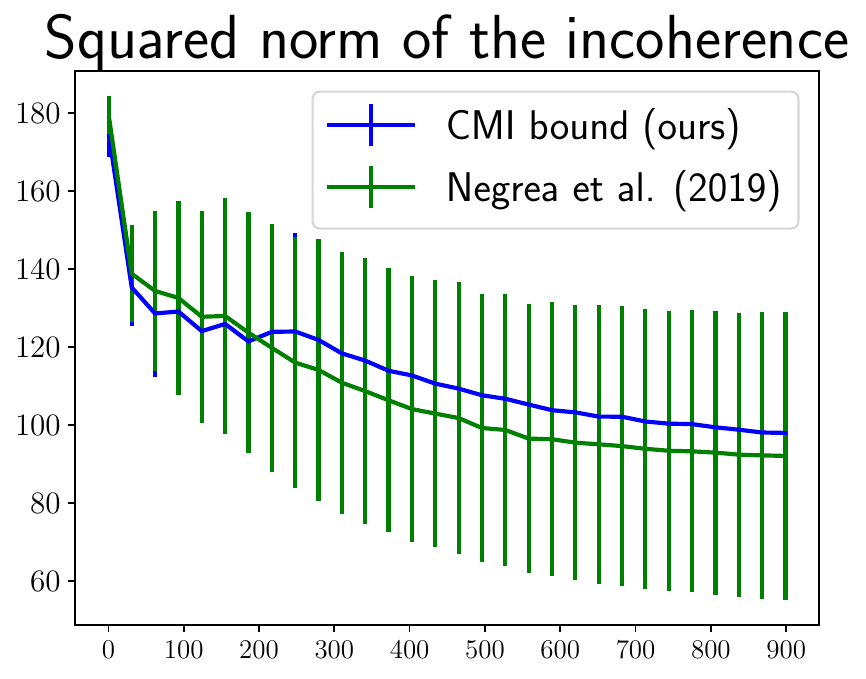}
\end{subfigure}
\begin{subfigure}[b]{.33\linewidth}
  \centering
  \includegraphics[width=1\linewidth, height=3.2cm]{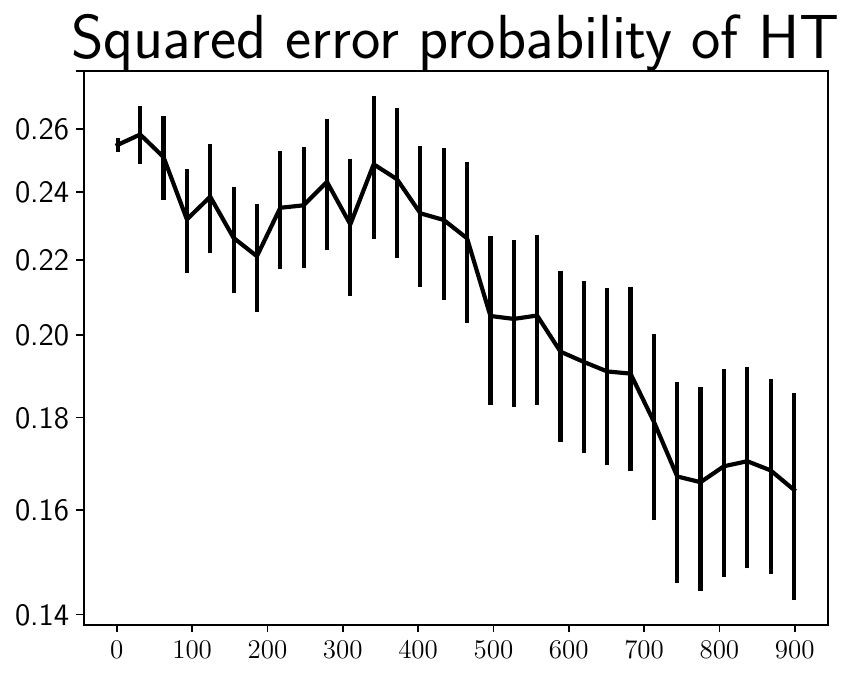}
\end{subfigure}
\newline
\begin{subfigure}[b]{.33\linewidth}
  \centering
  \includegraphics[width=1\linewidth, height=3.2cm]{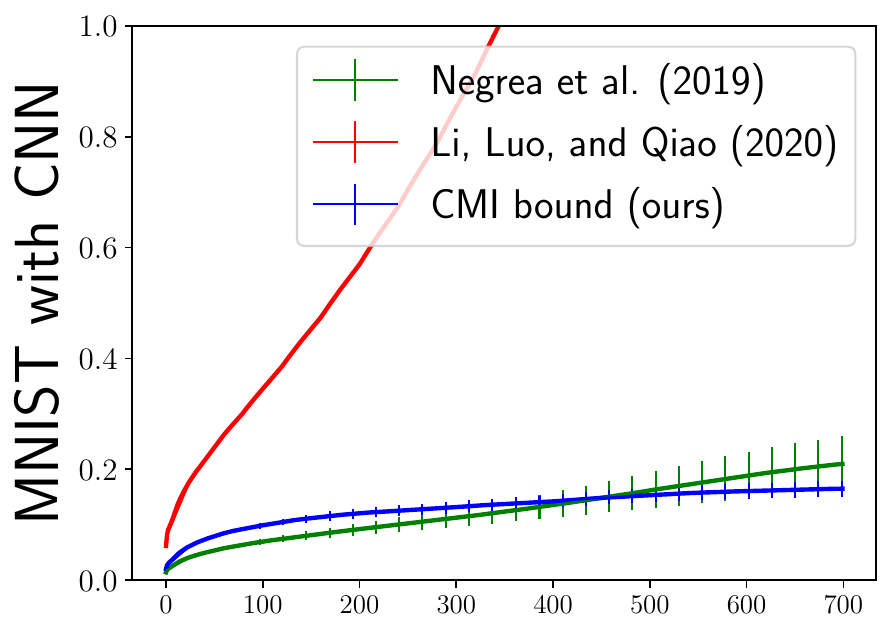}
\end{subfigure}%
\begin{subfigure}[b]{.33\linewidth}
  \centering
  \includegraphics[width=1\linewidth, height=3.2cm]{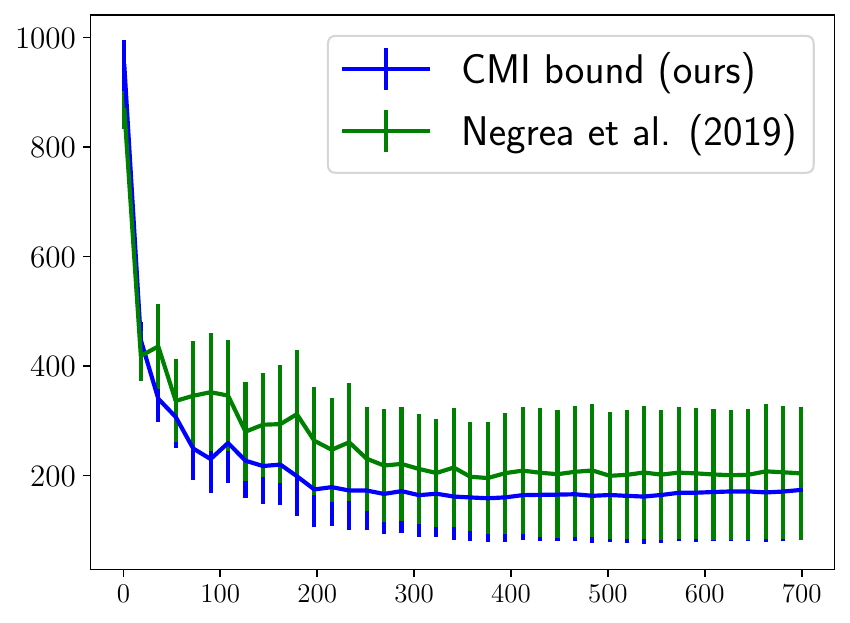}
\end{subfigure}
\begin{subfigure}[b]{.33\linewidth}
  \centering
  \includegraphics[width=1\linewidth, height=3.2cm]{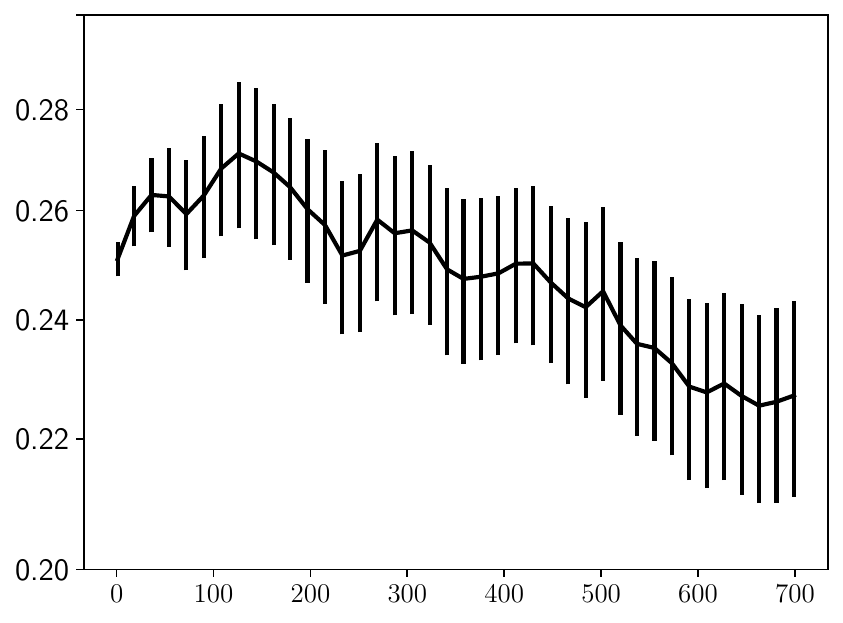}
\end{subfigure}
\newline
\begin{subfigure}[b]{.33\linewidth}
  \centering
  \includegraphics[width=1\linewidth, height=3.2cm]{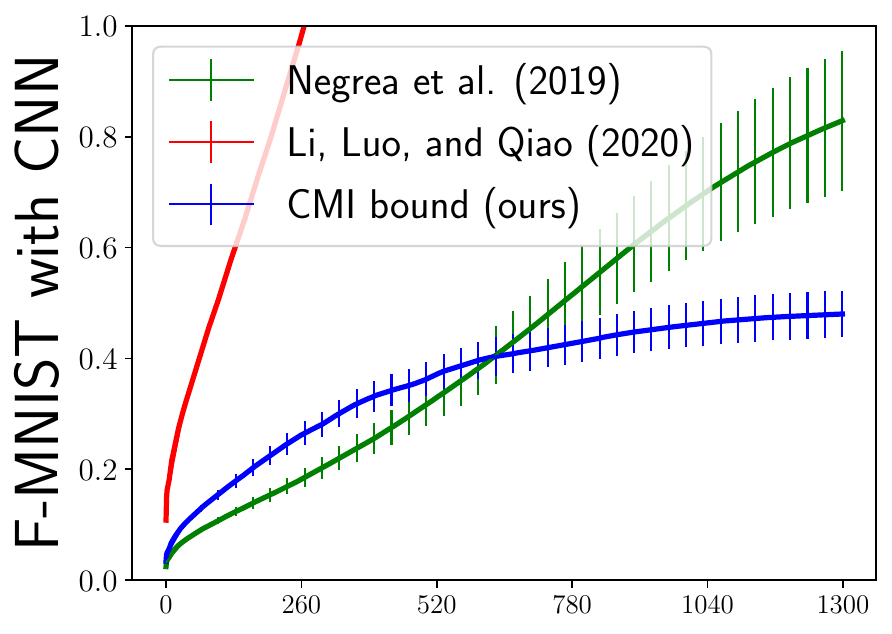}
\end{subfigure}%
\begin{subfigure}[b]{.33\linewidth}
  \centering
  \includegraphics[width=1\linewidth, height=3.2cm]{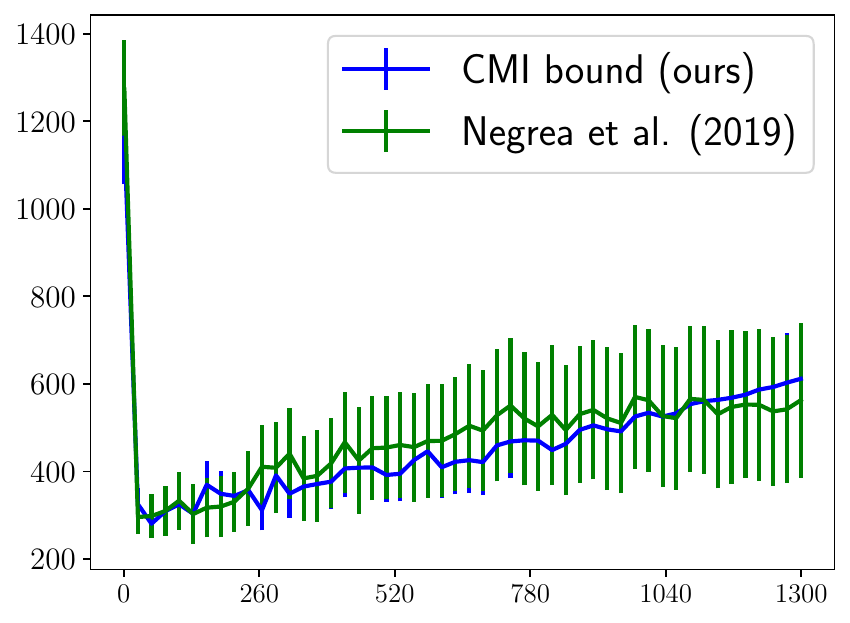}
\end{subfigure}
\begin{subfigure}[b]{.33\linewidth}
  \centering
  \includegraphics[width=1\linewidth, height=3.2cm]{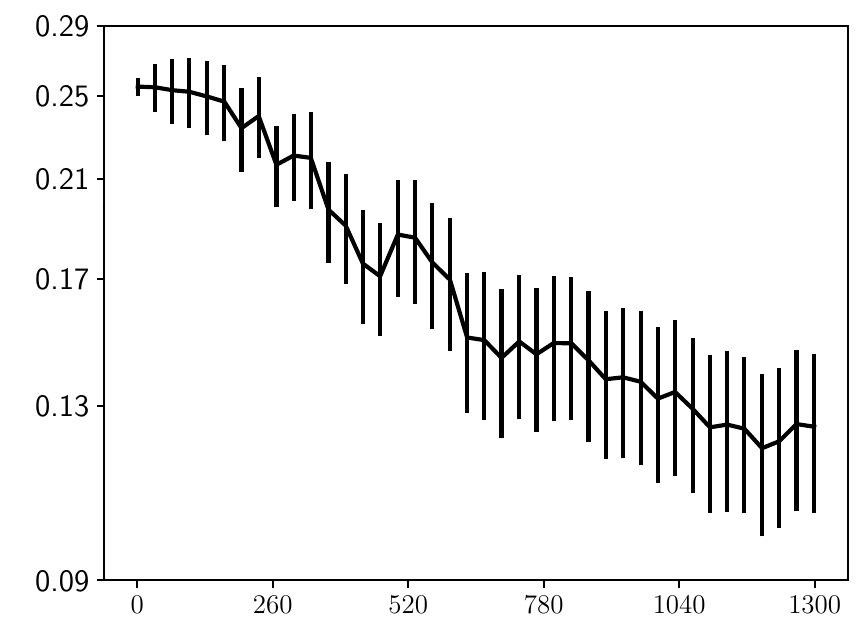}
\end{subfigure}
\newline
\begin{subfigure}[b]{.33\linewidth}
  \centering
  \includegraphics[width=1\linewidth, height=3.2cm]{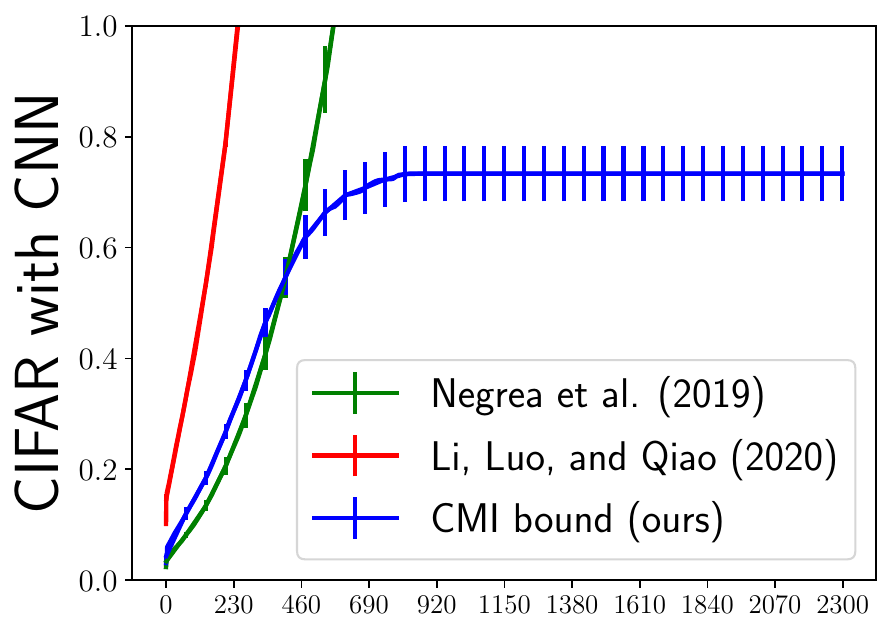}
\end{subfigure}%
\begin{subfigure}[b]{.33\linewidth}
  \centering
  \includegraphics[width=1\linewidth, height=3.2cm]{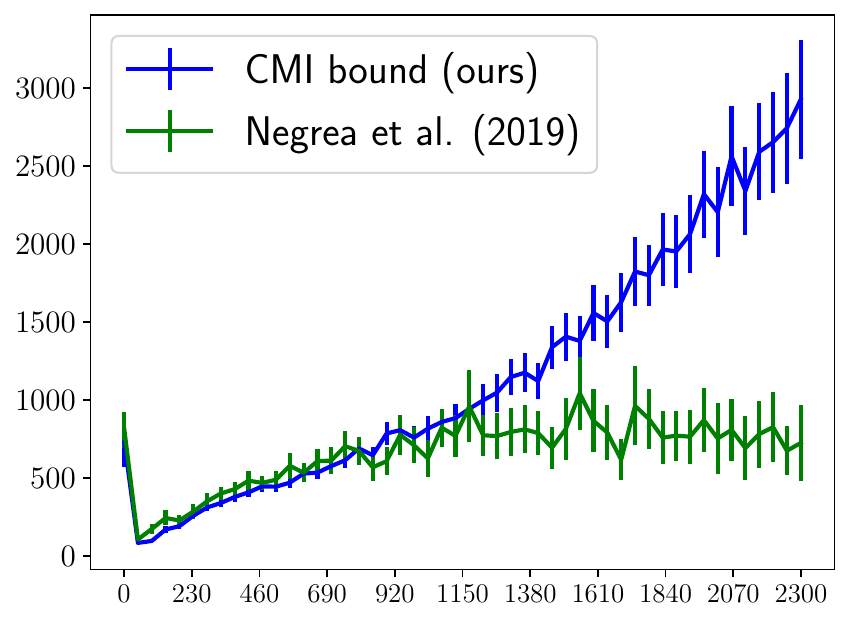}
\end{subfigure}
\begin{subfigure}[b]{.33\linewidth}
  \centering
  \includegraphics[width=1\linewidth, height=3.2cm]{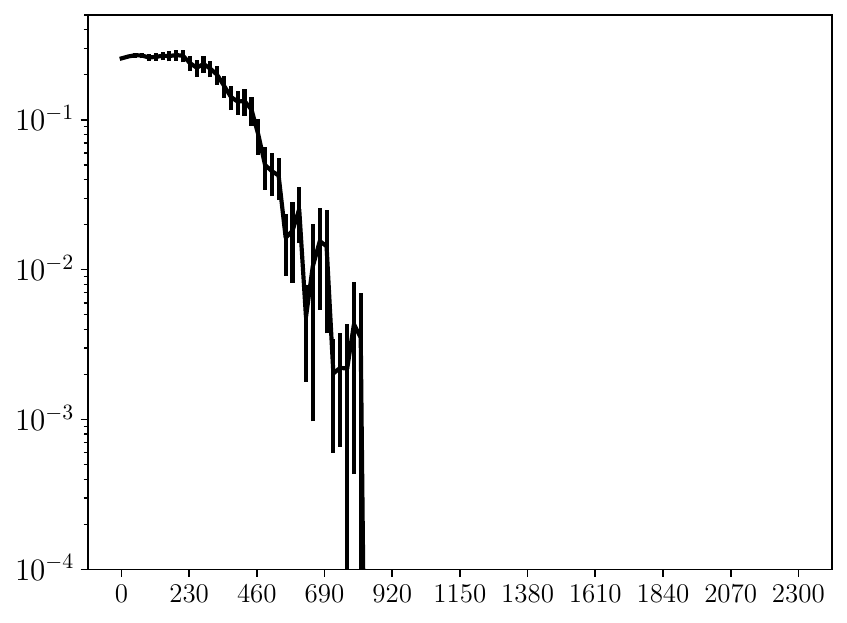}
\end{subfigure}
\caption{\small{Numerical results for various datasets and architectures. All the x-axes represent the training iteration. The plots in the first column depict a Monte Carlo estimate of our bounds with that of \citet{negrea2019information,li2019generalization}. The plots in the second column compare the mean of the \textit{training set incoherence} in \citep{negrea2019information} with the two-sample incoherence in our bound. Finally, the plots in the third column show the mean of the squared error probability of the hypothesis testing performed by the proposed prior in our bound. }}
\label{fig:the-plots}
\end{figure}
Using Monte Carlo (MC) simulation, we compared estimates of our expected generalization error bounds with estimates of the bound from \citep{negrea2019information,li2019generalization} for the MNIST \citep{MNIST}, CIFAR10 \citep{CIFAR10data}, and Fashion-MNIST \citep{xiao2017/online} datasets in \cref{fig:the-plots} and \cref{tbl:sum-results}. For all the plots we consider $\theta(x) = \frac{1}{2}(1+\mathrm{erf}(x))$ for our bound. Also, in the last row of \cref{tbl:sum-results}, we report the \textit{unbiased estimate} of our bound optimized over the choice of $\theta$ function. We plot the squared norm of the two sample incoherence and training set incoherence, as well as the squared error probability of the hypothesis test. 
\cref{fig:the-plots,tbl:sum-results} show that our bound is tighter, and remain non-vacuous after many more iterations. 
We also observe that the variances for MC estimates of our bound are smaller than those of \citet{negrea2019information}, and it is also smaller than   \citet{li2019generalization} for CIFAR10 and MNIST-CNN experiments. 
Moreover, we observe that the error probability of the hypothesis test decays with the number of iterations, which matches the intuition that, as one observes more noisy increments of the process, one is more able to determine which point is contributing to the gradient. For CIFAR10, $\|\incoh_{t}\|^2$ is large because the generalization gap is large. However, as mentioned in the beginning of this section, large $\|\incoh_{t}\|^2$ makes the hypothesis testing easier on subsequent iterations. For instance, after iteration $600$ the error is vanishingly small for CIFAR10 experiments which results in a plateau region in the bound. We can also observe the same phenomenon for the Fashion-MNIST experiment. 
This property distinguishes our bound from those in \citep{negrea2019information,li2019generalization}. 

Results for MNIST with CNN demonstrate that $\|\incoh_{t}\|^2$ and training set incoherence are close to each other. The reason behind this observations is that the generalization gap is small. Also, for this experiment the performance of the hypothesis testing is only slightly better than random guessing since the generalization gap is small, and it is difficult to distinguish the training samples from the test samples. This observation explains why our generalization bound is close to that of \citep{negrea2019information}. 
Nevertheless, the hypothesis testing performance improves with more training iterations, leading the two bounds to diverge, with our new bound performing better at later iterations. Finally, the scaling of our bound with respect to the number of iteration is tighter than in the bounds in \citep{negrea2019information,li2019generalization} as can be seen in \cref{fig:the-plots}.
\begin{table}
\small
\centering
\input{tables/table-gen-bounds.tex}
\vspace{3pt}
\caption{Summary of the results. The generalization bounds are reported at the end of training. }
\label{tbl:sum-results}
\end{table}
\newpage

%% file: tables/table-gen-bounds.tex
\renewcommand{\arraystretch}{1.2}
\begin{tabular}{c|c|c|c|c}
                                               & MNIST-MLP & MNIST-CNN & CIFAR10-CNN & FMNIST-CNN \\  \specialrule{3pt}{0pt}{0pt}
Training error                                 &   $4.33 \pm 0.01\% $        &    $2.59 \pm 0.01\% $       &    $9.39 \pm 0.36\% $      &    $7.96 \pm 0.03\% $ \\ \hline
Generalization error                           &   $0.88 \pm 0.01 \%$        &      $0.55 \pm 0.01\% $     &     $32.89 \pm 0.44\%$    &     $3.71 \pm 0.03\%$  \\ \specialrule{3pt}{0pt}{0pt}
\citet{negrea2019information} &   $67.93\pm 16.25 \% $      &    $20.98\pm 5.01 \% $       &   $4112.63 \pm 567.08 \%$      &   $82.89 \pm 12.64 \%$  \\ \hline
\citet{li2019generalization}  &   $600.29\pm 1.99 \% $        &   $245.03\pm 2.37 \% $        &        $20754.32 \pm 75.95 \%$  &        $598.62 \pm 3.21 \%$\\ \hline
\textbf{CMI (Ours)}                               &  $\mathbf{44.65\pm 4.27} \% $        &    $\mathbf{16.51\pm  1.41} \% $      &        $\mathbf{71.76 \pm 4.82} \% $  &        $\mathbf{48.01 \pm 4.22} \%$\\\hline
\textbf{CMI-OPT(Ours) }        &  $\mathbf{39.06\pm 5.52} \% $        &    $\mathbf{13.24\pm1.53} \% $      &        $\mathbf{63.00 \pm 5.97} \% $  &        $\mathbf{41.17 \pm5.85} \%$\\\hline
\end{tabular}

%% file: acknowledgments.tex
\subsection*{Acknowledgments}
The authors would like to thank Blair Bilodeau and Yasaman Mahdaviyeh for feedback on drafts of this work, and  Shiva Ketabi for helpful discussions on the implementation of the bounds. 

\subsection*{Funding}
MH is  supported by the Vector Institute. JN is supported by an NSERC Vanier Canada Graduate Scholarship, and by the Vector Institute.  DMR is supported by an NSERC Discovery Grant and an Ontario Early Researcher Award. This research was carried out in part while MH, JN, GKD, and DMR were visiting the Institute for Advanced Study. JN's visit to the Institute was funded by an NSERC Michael Smith Foreign Study Supplement.
Resources used in preparing this research were provided, in part, by the Province of Ontario, the Government of Canada through CIFAR, and companies sponsoring the Vector Institute \url{www.vectorinstitute.ai/partners}.

%% file: impact-statement.tex
\paragraph{Broader Impact}

This work builds upon the community's understanding of generalization error for machine learning methods. This has a positive impact on the scientific advancement of the field, and may lead to further improvements in our understanding, methodologies and applications of machine learning and AI. While there are not obvious direct societal implications of the present work, the indirect and longer term impact on society may be positive, negative or both depending on how, where and for what machine learning method that will have benefited from our research are used in the future.

%% file: appendices-1.tex
\section{CMI, Membership Attack, and Fano's Inequality}
\label{app:fano}

Let $\supersam{k},~\ssindex{k},$and $S$ as in \cref{def:cmi}. Consider the following hypothesis testing problem.
 Assume a decision maker observes $W$ and wishes to recover $\ssindex{k}$ by having access to the super-sample $\supersam{k}$. 
For any estimate $\hat{U}=\Psi\rbra[0]{W,\supersam{k}}$, we have the Markov chain
\*[
\ssindex{k} \rightarrow S \rightarrow W \rightarrow \widehat{U^{(k)}}.
\]
and so, combined with the fact that $\ssindex{k}$ is uniformly distributed over a set of size $k^n$,
we can invoke Fano's inequality to bound the error probability of the decision maker.
In particular,
\*[
 \inf_{\Psi} \ \Pr \left[\Psi\left(W,\supersam{k}\right)\neq \ssindex{k}\right]  
\geq 1-\frac{\minf{W;\ssindex{k}\vert \supersam{k}}+\log 2}{n\log k}.
\]
Hence, $\minf{W;\ssindex{k}\vert \supersam{k}}$ provides a lower bound on the hardness of the hypothesis testing problem, where one wants to identify the training sample given access to $\supersam{k}$ and $W$.

Some interpretation of our result is helpful. Consider an adversary who has access to the supersample $\supersam{k}$ and wishes to identify the training set that was used for the training after observing the output of a learning algorithm $W$. Our result here showed that the CMI upperbounds the success probability of \emph{every} adversary. Also, recall that the CMI upper bounds the expected generalization error. In the literature of data privacy in machine learning, this problem is known as Membership Attack \citep{shokri2017membership}, and it is empirically observed that a machine learning model leaks information about its training set when the generalization error is large \citep{shokri2017membership}. Our result in this section provides a formal connection between generalization and this specific membership attack problem.

\section{Matching the leading coefficient of \cref{thm:RZ_XR_bounds} with {$\SZB$}}
\label{apx:improved_const}
\begin{theorem}
\label{thm:improved_const}
Let $\SZB$ as defined in the introduction. Then, for $k>2$ 
\*[
\EGE \leq  \frac{\SZB}{\lambda^{\star}}+ \frac{ \exp(\frac{\lambda^{\star}}{n})-\frac{\lambda^{\star}}{n}-1}{\frac{\lambda^{\star}}{n}} \frac{k^3+7k^2-8k-16}{4(k^3-2k^2)} ,
\]
where 

$$
\lambda^{\star} =  n \mathfrak{W}\bigg(\big(\frac{4(k^3-2k^2) \SZB}{n(k^3+7k^2-8k-16)}-1\big)\exp(-1)\bigg) + n,
$$
and $\mathfrak{W}$ is the Lambert $W$ function.
\end{theorem} 
The proof is deferred to \cref{apx:proof-connections}. Here, we quantitatively compare \cref{thm:RZ_XR_bounds}, \cref{thm:cmi_main}, and \cref{thm:improved_const}, we consider the case that the output of $\Alg$ takes value in a finite set and  $k \to \infty$ . In this case, \cref{thm:improved_const} can rephrased as 
\[
\EGE &\leq \lim _{k \to \infty}  \frac{\SZB}{\lambda^{\star}}+ \frac{ \exp(\frac{\lambda^{\star}}{n})-\frac{\lambda^{\star}}{n}-1}{\frac{\lambda^{\star}}{n}}  \frac{k^3+7k^2-8k-16}{4(k^3-2k^2)} \nonumber \\
&= \frac{\IWS}{\lambda^{\star}_{\infty}}+ \frac{ \exp(\frac{\lambda^{\star}_{\infty}}{n})-\frac{\lambda^{\star}_{\infty}}{n}-1}{4\frac{\lambda^{\star}_{\infty}}{n}} 
\]
where  $\lambda^{\star}_{\infty} = n \mathfrak{W}\big((\frac{4 \IWS}{n}-1)\exp(-1)\big) + n$. In the next plot, we compare the values of the bounds in \cref{thm:RZ_XR_bounds}, \cref{thm:cmi_main}, and \cref{thm:improved_const} assuming $\IWS=1$. As seen, the bound in \cref{thm:improved_const} provides much tighter constant compared with \cref{thm:cmi_main}, and it matches with \cref{thm:RZ_XR_bounds}.
\begin{figure}[H]
    \centering
        \includegraphics[width=0.45\textwidth]{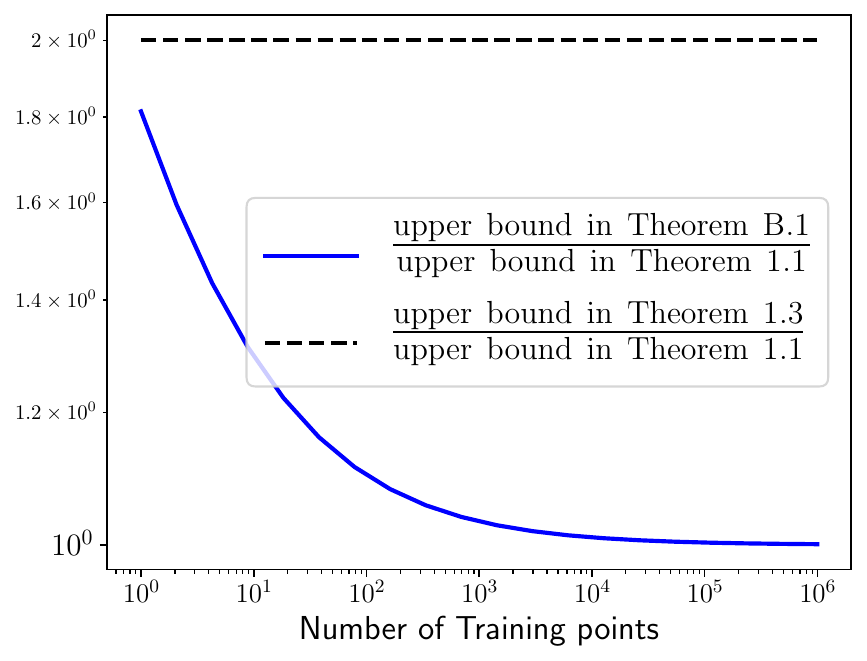}
        \caption{Comparison between constants of \cref{thm:RZ_XR_bounds}, \cref{thm:cmi_main}, and \cref{thm:improved_const} for the case $k \to \infty$.}
        \label{fig:const_bound}
\end{figure}

\section{Proofs of \cref{sec:mi-gen}}
\label{apx:proof-connections}
\begin{proof}[Proof of \cref{thm:cmi_lower_iomi}]
By the chain rule for the mutual information, we have
\[
\minf{W;\ssindex{k},\supersam{k}} 
 &= \minf{W;\tilde{Z}^{(k)}} + \minf{W;\ssindex{k}\vert \supersam{k}}.
\]
Since $S$ is $\sigma(\supersam{k}, \ssindex{k})$-measurable,
$\minf{W;\ssindex{k},\supersam{k}} = \minf{W;S,\ssindex{k},\supersam{k}}$.
But then $W$ is independent of $\supersam{k},\ssindex{k}$ given $S$,
hence $\minf{W;S,\ssindex{k},\supersam{k}} = \minf{W;S}$.
The result follows from the nonnegativity of mutual information.
\end{proof}

\begin{proof}[Proof of \cref{thm:k_inf}]
By \cref{thm:cmi_lower_iomi},
$
\minf{W;\ssindex{k}\vert \supersam{k}}=\minf{W;S} - \minf{W;\supersam{k}}.
$
Therefore, in order to prove the claim, we need to show $\lim_{k \to \infty} \minf{W;\supersam{k}}=0$ when $\minf{\ww;S}$ is finite.

\newcommand{\fww}{w}
Recall that $\Alg$ is a probability kernel from the space of tuples in $\dataspace$ to $\parspace$.
Assume $\parspace = \{\fww_1,\hdots,\fww_{m}\}$.
For each $l\in \range{m}$, let $\kappa_{l}(S)=\cPr{S}{W=\fww_l}$ and  $f_l: \dataspace^{kn}  \rightarrow [0,1]$ be a 
measurable function defined as
$$
f_l(\supersam{k}) = \frac{1}{k^n}\sum_{u \in \range{k}^n} \kappa_l(\supersam{k}_u).
$$ 
Letting $z,z'\in \dataspace^{kn}$ be two super-samples that only differ in one element, 
it is straightforward to see that
$$
|f_l(z) - f_l(z') |\leq \frac{1}{k}.
$$
Therefore, we can invoke McDiarmid's inequality to obtain
\[
\Pr[|f_l(\supersam{k})-\EE[f_l(\supersam{k})]|\geq \epsilon]\leq \exp \rbra[2]{-\frac{2k\epsilon^2 }{n}}.
\]
Also, we have $\EE[f_l(\supersam{k})]=\Pr[W=\fww_l]$ as each element of $\supersam{k}$ is IID. Hence, $f_l(\supersam{k}) \rightarrow \Pr[W=\fww_l]$ in probability as $k$ diverges.

By the definition of mutual information and KL divergence, 
\[
\minf{W;\supersam{k}} &= \EE[\KL{\cPr{\supersam{k}}{W}}{\Pr[W]}]  \nonumber\\
&= \EE \Big[\KL{\frac{1}{k^n}\sum_{u \in \range{k}^n} \cPr{\supersam{k}_u}{W} }{\Pr[W]}\Big] \nonumber\\
&= \EE\Big[ \sum_{l=1}^{m} \frac{1}{k^n}\sum_{u \in \range{k}^n} \kappa_l(\supersam{k}_u) \log \frac{\frac{1}{k^n}\sum_{u \in \range{k}^n} \kappa_l(\supersam{k}_u)}{\Pr[W=\fww_l]} \Big]  \nonumber \\
&= \sum_{l=1}^{m} \EE\Big[ f_l(\supersam{k}) \log \frac{f_l(\supersam{k})}{\Pr[W=\fww_l]}\Big] .
\]
Defining $\phi_l: [0,1] \rightarrow \Reals$ as $\phi_l(x)=x\log\frac{x}{\Pr[W=\fww_l]}$,
we have established %
\[
\minf{W;\supersam{k}}=\sum_{l=1}^{m} \EE \big[\phi_l\big(f_l(\supersam{k})\big)\big].
\]
Note that $\phi_l$ is a continuous and bounded function. 
By a standard result \citep[Thm.~2.3.4]{durrett2019probability},
 $f_l(\supersam{k}) \rightarrow \Pr[W=\fww_l]$ in probability implies that 
 $$\EE \big[\phi_l\big(f_l(\supersam{k})\big)\big] \rightarrow \EE \big[\phi_l\big(\Pr[W=\fww_l]\big)\big]=0,$$ 
 as  $k$ goes to infinity.  
 Using this, we conclude that $\minf{W;\supersam{k}}\rightarrow 0$ as $k$ diverges, as was to be shown.
\end{proof}
\begin{proof}[Proof of \cref{thm:improved_const}]
For any $k \in \Naturals$ define 
$$
\rho^{(k)}_i(m)= \begin{cases} -1 &\mbox{if } m = i, \\
\frac{1}{k-1} & \mbox{otherwise }  \end{cases}. 
$$
where $m$ and  $i \in \range{k}$.
Consider random variables $\supersam{k}$, $U$, and $W$ as in the definition of $\SZB$. Also, let  $\hat{U} \eqdist U$ and $\hat{U} \indep (\supersam{k},W,U)$. Let $f: \dataspace^{ kn} \times [k]^{ n}\times  \parspace \rightarrow [-1,1] $ be given by
$$
f(\tilde{z}^{(k)},u,w)= \frac{1}{n}\sum_{j=1}^{n}\sum_{i=1}^{k} \rho^{(k)}_i(u_j)\loss(w,z_{i,j}).
$$
Then, by the Donsker--Varadhan variational formula \citep[Prop.~4.15]{boucheron2013concentration} of the KL divergence we obtain
\[
\SZB&=\EE[\KL{\cPr{\supersam{k},W}{U}}{\Pr[U]}] \nonumber\\
&\geq \lambda \EE[f(\supersam{k},U,W)]  - \EE \log \cEE{\supersam{k},W}{[ \exp(\lambda f(\supersam{k},\hat{U},W))]}
\label{eq:dv-const-cmik}
\]
It is straightforward to see that the first term in the RHS of \cref{eq:dv-const-cmik} is  $\EE[f(\supersam{k},U,W)]=\EGE$. In what follows we analyze the second term in the RHS of \cref{eq:dv-const-cmik}. We begin with
\[
\EE \log \cEE{\supersam{k},W}{[ \exp(\lambda f(\supersam{k},\hat{U},W))]} &=  \EE \log \cEE{\supersam{k},W}{[ \prod_{j=1}^{n} \exp( \frac{\lambda}{n} \sum_{i=1}^{k} \rho^{(k)}_i(\hat{U}_j)\loss(W,Z_{i,j})]}\\
&=\EE \log \prod_{j=1}^{n} \cEE{\supersam{k},W} {\exp\big(\frac{\lambda}{n}\sum_{i=1}^{k}\rho^{(k)}_{i}(\hat{U}_j)\loss(W,Z_{i,j})\big)}\nonumber \\
&\leq \EE \sum_{j=1}^{n}\big( \exp(\frac{\lambda}{n})-\frac{\lambda}{n}-1\big) \cEE{\supersam{k},W}{[\sum_{i=1}^{k}\rho^{(k)}_i(\hat{U}_j)\loss(W,Z_{i,j})]^2} \label{eq:step2-const-cmik}
\]  
The first step follows from the independence of the elements in $\hat{U}$. The last inequality is obtained by the Bennet's inequality \citep[Thm.~2.9]{boucheron2013concentration} on the moment generating function and the fact that the elements of $\hat{U}$ are independent of $(\supersam{k},W)$. Also, we have used  
$\cEE{\supersam{k},W}{\sum_{i=1}^{k}\rho^{(k)}_i(\hat{U}_j)\loss(W,Z_{i,j})}=0$ 
since $\EE \rho^{(k)}_i(\hat{U}_j)=0$ and  $|\sum_{i=1}^{k}\rho^{(k)}_i(\hat{U}_j)\loss(W,Z_{i,j})|\leq 1$ \as. For a fixed $j$, from \cref{eq:step2-const-cmik} we obtain
\[
&\EE[\sum_{i=1}^{k}\rho^{(k)}_i(\hat{U}_j)\loss(W,Z_{i,j})]^2=  \frac{1}{k}\EE  \sum_{\tilde{i}=\range{k}}\big[\frac{1}{k-1} \sum_{i\in \range{k},i\neq \tilde{i}} \loss(W,Z_{i,j}) - \loss(W,Z_{\tilde{i},j})\big]^2 \label{eq:def_rho} \\
&= \frac{1}{k^2} \EE \bigg[\sum_{u_j \in \range{k}} \cEE{\supersam{k},U_j=u_j} \sum_{\tilde{i}=\range{k}} \big[\frac{1}{k-1} \sum_{i\in \range{k},i\neq \tilde{i}} \loss(W,Z_{i,j}) - \loss(W,Z_{\tilde{i},j})\big]^2\bigg]\label{eq:constant_cmik_law_iterated}\\
& = \frac{1}{k^2} \EE \bigg[\sum_{u_j \in \range{k}} \cEE{\supersam{k},U_j=u_j} \Big[\sum_{\tilde{i}=\range{k}, \tilde{i}\neq u_j} \big[\frac{1}{k-1}\loss(W,Z_{u_j,j})+\frac{1}{k-1} \sum_{i\in \range{k},i \neq \{\tilde{i},u_j\}} \loss(W,Z_{i,j}) - \loss(W,Z_{\tilde{i},j})\big]^2 \nonumber \\
&\hspace{2cm} + \big[\frac{1}{k-1} \sum_{i\in \range{k},i\neq \tilde{i}} \loss(W,Z_{i,j}) - \loss(W,Z_{u_j,j})\big]^2\Big]\bigg] \label{eq:constant_cmik_manipulation}\\
&\leq  \frac{1}{k^2} \EE \bigg[\sum_{u_j \in \range{k}} \cEE{\supersam{k},U_j=u_j} \Big[\sum_{\tilde{i}=\range{k}, \tilde{i}\neq u_j} \big[ \big(\frac{1}{k-1} + \frac{1}{k-1} \sum_{i\in \range{k},i \neq \{\tilde{i},u_j\}} \loss(W,Z_{i,j})\big)^2+ \loss(W,Z_{\tilde{i},j})^2  \nonumber \\
&\hspace{2cm} -2\loss(W,Z_{\tilde{i},j}) \frac{1}{k-1} \sum_{i\in \range{k},i \neq \{\tilde{i},u_j\}} \loss(W,Z_{i,j}) \big]+ \big[ \big[\frac{1}{k-1} \sum_{i\in \range{k},i\neq \tilde{i}} \loss(W,Z_{i,j})\big]^2+1\big]\Big]\bigg]  \label{eq:constant_cmik_loss01}\\
&= \frac{1}{k^2} \sum_{u_j\in \range{k}}  \EE \bigg[ \sum_{\tilde{i} \in \range{k},u_j \neq \tilde{i}} \cEE{W}{ \big[ \big(\frac{1}{k-1} + \frac{1}{k-1} \sum_{i\in \range{k},i \neq \{\tilde{i},u_j\}} \loss(W,Z_{i,j})\big)^2+ \loss(W,Z_{\tilde{i},j})^2 } \nonumber\\
&\hspace{2cm} -2\loss(W,Z_{\tilde{i},j}) \frac{1}{k-1} \sum_{i\in \range{k},i \neq \{\tilde{i},u_j\}} \loss(W,Z_{i,j}) \big] + \cEE{W}{\big[\big[\frac{1}{k-1} \sum_{i\in \range{k},i\neq \tilde{i}} \loss(W,Z_{i,j})\big]^2+1\big]}\bigg]. \label{eq:constant_cmik_order_expectation}
\]
Here, \cref{eq:def_rho} is obtained by taking the expectation over $\hat{U}_j$, the definition of $\rho^{(k)}_i(\hat{U}_j)$, and $\hat{U}\indep (\supersam{k},W)$. Then, \cref{eq:constant_cmik_law_iterated} is by the law of iterated expectations. Specifically, we condition on $U_j$, and recall that based on the \cref{def:cmi} the $j$-th training sample is $Z_{U_j,j}$.  Step \cref{eq:constant_cmik_manipulation} is by some manipulations.  \cref{eq:constant_cmik_loss01} is obtained by 
\*[
&\big[\frac{1}{k-1}\loss(W,Z_{u_j,j})+\frac{1}{k-1} \sum_{i\in \range{k},i \neq \{\tilde{i},u_j\}} \loss(W,Z_{i,j}) - \loss(W,Z_{\tilde{i},j})\big]^2 \leq \bigg( \frac{1}{k-1}+  \frac{1}{k-1} \sum_{i\in \range{k},i \neq \{\tilde{i},u_j\}} \loss(W,Z_{i,j})\bigg)^2\\
&\hspace{5cm} +\loss(W,Z_{\tilde{i},j})^2 -\loss(W,Z_{\tilde{i},j})\frac{2}{k-1} \sum_{i\in \range{k},i \neq \{\tilde{i},u_j\}} \loss(W,Z_{i,j}) ,
\]
and
\*[
\big[\frac{1}{k-1} \sum_{i\in \range{k},i\neq \tilde{i}} \loss(W,Z_{i,j}) - \loss(W,Z_{u_j,j})\big]^2 \leq 1+ \bigg(\frac{1}{k-1} \sum_{i\in \range{k},i\neq \tilde{i}=u_j} \loss(W,Z_{i,j})\bigg)^2.
\]
Finally last step is obtained by changing the order of the expectation over $W$ and $\supersam{k}$. Then, we can simplify \cref{eq:constant_cmik_order_expectation} by considering
\[
&\cEE{W}{ \big[ \big(\frac{1}{k-1} + \frac{1}{k-1} \sum_{i\in \range{k},i \neq \{\tilde{i},u_j\}} \loss(W,Z_{i,j})\big)^2+ \loss(W,Z_{\tilde{i},j})^2 } -2\loss(W,Z_{\tilde{i},j}) \frac{1}{k-1} \sum_{i\in \range{k},i \neq \{\tilde{i},u_j\}} \loss(W,Z_{i,j}) \big] \nonumber\\
&\leq  \Risk{\Dist}{W}^2  \frac{-k^2+k+2}{(k-1)^2}  + \Risk{\Dist}{W}\frac{k^2+k-5}{(k-1)^2}+\frac{1}{(k-1)^2}\triangleq A_{1}(k,\Risk{\Dist}{W})  \label{eq:term1_cmik} \\
&\cEE{W}{\big[\big[\frac{1}{k-1} \sum_{i\in \range{k},i\neq \tilde{i}} \loss(W,Z_{i,j})\big]^2+1\big]} \leq \Risk{\Dist}{W}^2  \frac{k-2}{k-1}  + \Risk{\Dist}{W}\frac{1}{k-1}+1 \triangleq A_{2}(k,\Risk{\Dist}{W}). \label{eq:term2_cmik} 
\]
Note that in \cref{eq:term1_cmik} and \cref{eq:term2_cmik}, $W$ and $Z_{i,j}$s are independent, therefore $\cEE{W}{[\loss(W,Z_{i,j})]}=\Risk{\Dist}{W}$. Also, we have $\text{Var}^{W}[\loss(W,Z_{i,j})]\leq \Risk{\Dist}{W}(1-\Risk{\Dist}{W})$ because Bernoulli random variable has the largest variance among the $[0,1]$-bounded random variables with the same mean. Plugging \cref{eq:term1_cmik} and \cref{eq:term2_cmik} into \cref{eq:constant_cmik_order_expectation} we obtain
\[
\label{eq:constant_cmi_k_step_before_der}
\EE[\sum_{i=1}^{k}\rho^{(k)}_i(U_j)\loss(W,Z_{i,j})]^2 \leq \frac{1}{k}  \EE [(k-1)A_{1}(k,\Risk{\Dist}{W}) + A_{2}(k,\Risk{\Dist}{W})].
\]
We can upper bound the LHS of \cref{eq:constant_cmi_k_step_before_der} by maximizing it over $\Risk{\Dist}{W}$ to obtain
\*[
\frac{\partial [(k-1)A_{1}(k,R) + A_{2}(k,R)]}{\partial R}=0 \Rightarrow R^{\star}=\frac{k^2+k-4}{2k^2-4k}.
\]
We can plug the expression for $R^{\star}$ into \cref{eq:constant_cmi_k_step_before_der} to get
\[
\EE[\sum_{i=1}^{k}\rho^{(k)}_i(U_j)\loss(W,Z_{i,j})]^2 &\leq \frac{1}{k}  \EE [(k-1)A_{1}(k,R^{\star}) + A_{2}(k,R^{\star})]\nonumber\\
&=  \frac{k^3+7k^2-8k-16}{4(k^3-2k^2)} \label{eq:cmik-const-bound-eachterm}.
\]
Then, plugging \cref{eq:cmik-const-bound-eachterm} into \cref{eq:dv-const-cmik} yields
\[
\label{eq:dv-simpilify}
\inf_{\lambda \geq 0} \frac{\SZB}{\lambda}+ \frac{ \exp(\frac{\lambda}{n})-\frac{\lambda}{n}-1}{\frac{\lambda}{n}} \frac{k^3+7k^2-8k-16}{4(k^3-2k^2)} \geq \EGE.
\]
Finally, the closed form solution of \cref{eq:dv-simpilify} is given by
\*[
&\frac{\partial [\frac{\SZB}{\lambda}+ \frac{ \exp(\frac{\lambda}{n})-\frac{\lambda}{n}-1}{\frac{\lambda}{n}} \frac{k^3+7k^2-8k-16}{4(k-2)k^2}]}{\partial \lambda}=0 \Rightarrow \\
&  \lambda^{\star} = n \mathfrak{W}\bigg((\frac{4(k^3-2k^2) \SZB}{n(k^3+7k^2-8k-16)}-1)\exp(-1)\bigg) + n,
\]
which is the desired result.
\end{proof}
\section{Proofs of  \cref{sec:gen-bounds}}
\label{apx:proof-bounds}
\begin{proof}[Proof of \cref{thm:gen-bound-rndsub}] 
With $k=2$, recall from the introduction
\*[
\supersam{2}= \begin{pmatrix} 
    Z_{1,1} &   \dots & Z_{1,n} \\
    Z_{2,1} & \dots  & Z_{2,n} 
    \end{pmatrix}  \sim  \Dist^{\otimes 2n},
\]
and
$U=\left(U_1,\hdots,U_n\right)\in \{1,2\}^{n}$ where $U_i$s are IID, and the marginal distribution follows $U_i \dist \mathrm{Bern}\left(\frac{1}{2}\right)$ for $i\in \range{n}$. 
Furthermore, recall $\trainset = \left\lbrace Z_{U_1,1},\hdots,Z_{U_n,n} \right\rbrace$. 
The expected generalization error can be written as
\begin{align}
\EE \left[\Risk{\Dist}{W} - \EmpRisk{S}{W}\right] &= \EE \Big[ \frac{1}{n} \sum_{i=1}^{n} (-1)^{U_i} \left( \loss\left(Z_{1,i},W\right)-\loss\left(Z_{2,i},W\right)\right) \Big]\\
& =  \EE \Big[\frac{1}{m} \sum_{i=1}^{m} (-1)^{U_{J_i}} \left( \loss\left(Z_{1,J_i},W\right)-\loss\left(Z_{2,J_i},W\right)\right)\Big],
\end{align}
where the last equality follows because $J$ is independent of $U_J$, $\supersam{2}$, and $W$.

Define $\tilde{W}$, $\tilde{U_J}$, and $\tilde{J}$ such that  $(W,U_J,J,\supersam{2}) \eqdist (\tilde{W},\tilde{U_J},\tilde{J},\supersam{2})$, and $\tilde{W}$, $\tilde{U_J}$, and $\tilde{J}$ are independent given $\supersam{2}$. By the Donsker--Varadhan variational formula \citep[Prop.~4.15]{boucheron2013concentration} and the disintegration theorem \citep[Thm.~6.4]{FMP2}, for all measurable functions $g$ in $\mathcal{G}$, i.e., the class of all functions $g$ such that $\left(\cPr{\tilde{Z}^{(2)}}{\tilde{W}} \otimes \cPr{\tilde{Z}^{(2)}}{\tilde{U_J}} \otimes \cPr{\supersam{2}}{\tilde{J}} \right)(\exp g)< \infty$, with probability one we have
\begin{align}
\dminf{\tilde{Z}^{(2)}}{W,J;U_J}
     &= \KL{ \cPr{\supersam{2}}{W,J,U_J} }{\cPr{\supersam{2}}{\tilde{W}} \otimes \cPr{\supersam{2}}{\tilde{\tilde{J}}} \otimes \cPr{\tilde{Z}^{(2)}}{\tilde{U_J}} }\\
  &=\sup_{g \in \mathcal{G}} \cPr{\tilde{Z}^{(2)}}{W,J,U_J} (g) - \log \left[ \left(\cPr{\supersam{2}}{\tilde{W}} \otimes \cPr{\supersam{2}}{\tilde{J}}  \otimes \cPr{\supersam{2}}{\tilde{U_J}} \right)(\exp g)\right] \label{eq:dv-gen-bound}.
\end{align}
Define $f(w,j,u_j)\triangleq  \frac{\lambda}{m}\sum_{i=1}^{m} (-1)^{u_{j_i}} \left( \loss\left(z_{1,j_i},w\right)-\loss\left(z_{2,j_i},w\right)\right)$ where $\lambda \ge 0$. By \cref{eq:dv-gen-bound}, we can write
\begin{align}
\dminf{\tilde{Z}^{(2)}}{W,J;U_J} \geq \cPr{\tilde{Z}^{(2)}}{W,U_J,J} (f) - \log \Big[ \left(\cPr{\tilde{Z}^{(2)}}{\tilde{W}} \otimes \cPr{\tilde{Z}^{(2)}}{\tilde{U_J}} \otimes \cPr{\supersam{2}}{\tilde{J}} \right)(\exp f)\Big].
\label{eq:dv-gen-bound2}
\end{align}
Considering the second term of the RHS of \cref{eq:dv-gen-bound2}, Hoeffding's lemma implies that
\begin{align}
&\left(\cPr{\supersam{2}}{\tilde{W}} \otimes \cPr{\tilde{Z}^{(2)}}{\tilde{U_J}} \otimes \cPr{\supersam{2}}{\tilde{J}} \right)(\exp f)\\
 &= \cEE{\supersam{2}}  \exp\Big( \frac{\lambda}{m} \sum_{i=1}^{m}  (-1)^{\tilde{U}_{J_i}} \left( \loss\left(Z_{1,\tilde{J_i}},\tilde{W}\right)-\loss\left(Z_{2,\tilde{J_i}},\tilde{W}\right)\right)\Big)\\
&= \cEE{\supersam{2}}\cEE{\supersam{2},\tilde{W},\tilde{J}} \prod_{i=1}^{m} \exp\Big( \frac{\lambda}{m}  (-1)^{\tilde{U}_{J_i}} \left( \loss\left(Z_{1,\tilde{J_i}},\tilde{W}\right)-\loss\left(Z_{2,\tilde{J_i}},\tilde{W}\right)\right)\Big)\\
&\leq  \cEE{\supersam{2}}\cEE{\supersam{2},\tilde{W},\tilde{J}}  \prod_{i=1}^{m} \exp \Big(\frac{\lambda^2 \left(\loss\left(Z_{1,\tilde{J_i}},\tilde{W}\right)-\loss\left(Z_{2,\tilde{J_i}},\tilde{W}\right)\right)^2}{2m^2}\Big)\\
&\leq\exp \left(\frac{\lambda^2}{2m}\right), \label{eq:rhs-2-prod}
\end{align} 
where we use the fact that  
\*[
\left(\cPr{\tilde{Z}^{(2)}}{\tilde{W}} \otimes \cPr{\tilde{Z}^{(2)}}{\tilde{U_J}} \otimes \cPr{\supersam{2}}{\tilde{J}} \right)(f)=0.
\]
Substituting the bound in \cref{eq:rhs-2-prod} into \cref{eq:dv-gen-bound2}, rearranging and taking expectations, we obtain
\begin{align}
\EE \frac{1}{m} \sum_{i=1}^{m} (-1)^{U_{J_i}} \left( \loss\left(Z_{1,J_i},W\right)-\loss\left(Z_{2,J_i},W\right)\right) &\leq \EE \inf_{\lambda\geq 0} \frac{\dminf{\supersam{2}}{W,J;U_J}+\frac{\lambda^2}{2m}}{\lambda} \\
&= \EE\sqrt{\frac{2}{m} \dminf{\supersam{2}}{W,J;U_J} }. \label{eq:gen_bound_-1}
\end{align}

Moreover, we have $\as$
\[
\label{eq:2mi_eq}
\dminf{\supersam{2}}{J,W;U_J}-\underbrace{\dminf{\supersam{2}}{J;U_J}}_{0}=\dminf{\supersam{2}}{W;U_J\vert J}.
\]
Here, $\dminf{\supersam{2}}{J;U_J}=0$ since $J$ is independent of $U_J$ given $\supersam{2}$. Plugging \cref{eq:2mi_eq} into \cref{eq:gen_bound_-1}, we obtain the desired result.
\end{proof}

\begin{proof}[Proof of \cref{thm:gen-rndsub-prop}]
Consider
\[
\minf{W;U_{J^{(m_1)}}\vert  \supersam{2},J^{(m_1)} } &= \entr{U_{J^{(m_1)}}\vert J^{(m_1)},\supersam{2}} - \entr{U_{J^{(m_1)}}\vert J^{(m_1)},\supersam{2},W} \label{eq:zero_step_indep_j_Z} \\
&= \frac{1}{{n\choose m_1}} \sum_{K_1 \in \range{n}_{m_1}} \entr{U_{K_1}\vert \supersam{2}} - \entr{U_{J^{(m_1)}}\vert J^{(m_1)},\supersam{2},W} \label{eq:first_step_indep_j_Z}\\
&=\frac{1}{{n\choose m_1}} \sum_{K_1 \in \range{n}_{m_1}} \entr{U_{K_1}\vert \supersam{2}} - \frac{1}{{n\choose m_1}} \sum_{K_1 \in \range{n}_{m_1}} \entr{U_{K_1}\vert \supersam{2},W}. \label{eq:second_step_indep_j_Z}
\]
\cref{eq:first_step_indep_j_Z} follows because $\supersam{2} \indep J^{(m_1)}$, 
while \cref{eq:second_step_indep_j_Z} follows because the event $\{J^{(m_1)}=K_1\}$ is independent of  $(W, U_{K_1},\supersam{2})$. Then
\[
\frac{1}{m_1}\minf{W;U_{J^{(m_1)}}\vert  \supersam{2},J^{(m_1)} } &= \frac{1}{m_1 {n\choose m_1} }  \sum_{K_1 \in \range{n}_{m_1} } [\entr{U_{K_1}} - \entr{U_{K_1}\vert W,\supersam{2}}]\label{eq:result-above} \\
&= \frac{1}{n}  \entr{U} - \frac{1}{m_1 {n\choose m_1} }\sum_{K_1 \in \range{n}_{m_1} } \entr{U_{K_1}\vert W,\supersam{2}} \label{eq:u-iid}\\
&= \frac{1}{m_2 {n\choose m_2} } \sum_{K_2 \in \range{n}_{m_2} } \entr{U_{K_2}} - \frac{1}{m_1 {n\choose m_1} } \sum_{K_1 \in \range{n}_{m_1} }\entr{U_{K_1}\vert W,\supersam{2}}\\
&\leq  \frac{1}{m_2 {n\choose m_2} } \sum_{K_2 \in \range{n}_{m_2} } [\entr{U_{K_2}} - \entr{U_{K_2}\vert W,\supersam{2}}] \label{eq:han-ineq}\\
&=\frac{1}{m_2}\minf{W;U_{J^{(m_2)}}\vert \supersam{2} ,J^{(m_2)} }.
\]
\cref{eq:result-above} follows from \cref{eq:second_step_indep_j_Z} and the fact that $U \indep \supersam{2}$,
while \cref{eq:u-iid} follows from each element of $U$ being IID. 
\cref{eq:han-ineq} follows from \cref{lem:han-lemma}, which is a modified version of the Han's inequality \citep{te1978nonnegative}. 
Finally, the last step follows from using the same line of reasoning as in \cref{eq:zero_step_indep_j_Z} to \cref{eq:second_step_indep_j_Z}. 

Having established \cref{eq:monoton-wrt-m}, the claim follows from
\begin{align}
\EE \sqrt {\frac{2}{m}\dminf{\tilde{Z}^{(2)}}{W;U_J\vert J}} &\leq \sqrt{ \frac{2}{m} \minf{W;U_J\vert \supersam{2},J }} \label{eq:jensen}\\
 &\leq \sqrt{ \frac{2}{n}\minf{W;U\vert \supersam{2} }}, \label{eq:final-tight}
\end{align}
where \cref{eq:jensen} is Jensen's inequality, and \cref{eq:final-tight} is the direct application of \cref{eq:monoton-wrt-m} with $m_1=m$ and $m_2=n$. This proves the desired result. 
\end{proof}

\begin{proof}[Proof of \cref{thm:gen-bu-style}]
By the Donsker--Varadhan variational formula \citep[Prop.~4.15]{boucheron2013concentration} and the disintegration theorem \citep[Thm.~6.4]{FMP2}, with probability one, for all measurable functions $g$
such that $\rbra[1]{\cPr{\tilde{Z}^{(2)}}{\tilde{W}} \otimes \cPr{\tilde{Z}^{(2)}}{\tilde{U_i}} }(\exp g)< \infty$,
we have
\begin{align}
\dminf{\supersam{2}}{U_i,W}&=\KL{\cPr{\supersam{2}}{U_i,W }}{\cPr{\supersam{2}}{\tilde{U_i}}\otimes \cPr{\supersam{2}}{\tilde{W}}} \\ 
&\geq \cPr{\supersam{2}}{U_i,W }[ g(W,\supersam{2},U_i)] - \log \cPr{\supersam{2}}{\tilde{U}_i}\otimes \cPr{\supersam{2}}{\tilde{W}}[ \exp(g(\tilde{W},\supersam{2},\tilde{U_i}))]
\label{eq:dv-lemma-bu}
\end{align}
where $\big(W,U_i,\supersam{2}\big)\equaldist \big(\tilde{W},\tilde{U_i},\supersam{2}\big)$ and $\tilde{W} \indep \tilde{U_i} \given \supersam{2}$.
For $i \in \range{n}$, let
\*[
g_i\big(W,\supersam{2},U_i\big)\triangleq \lambda \left(-1\right)^{U_i} \left( \loss\left(Z_{1,i},W\right)-\loss\left(Z_{2,i},W\right)\right).
\]
Hoeffding's lemma implies that
\begin{align}
\cPr{\supersam{2}}{\tilde{U}_i}\otimes \cPr{\supersam{2}}{\tilde{W}}[\exp ( g_i(\tilde{W},\supersam{2},\tilde{U_i}))] \leq \exp \left(\frac{\lambda^2}{2}\right),
\end{align} 
where in the last line we have used $g_i \in [-\lambda,\lambda]$ a.s.
From (\ref{eq:dv-lemma-bu}), we obtain
\begin{align}
\cEE{\supersam{2}} \left(-1\right)^{U_i} \left( \loss\left(Z_{1,i},W\right)-\loss\left(Z_{2,i},W\right)\right) &\leq \inf_{\lambda\geq 0} \frac{\KL{\cPr{\supersam{2}}{U_i,W }}{\cPr{\supersam{2}}{U_i}\otimes \cPr{\supersam{2}}{W}}+\frac{\lambda^2}{2}}{\lambda}\\
&= \sqrt{2 \dminf{\supersam{2}}{W;U_i} } .
\end{align}
Then, averaging over $i$ and taking expectations, 
\[
\EE \left[\Risk{\Dist}{W} - \EmpRisk{S}{W}\right] &= \EE \frac{1}{n}\sum_{i=1}^{n} \cEE{\supersam{2}} \left(-1\right)^{U_i} \left( \loss\left(Z_{1,i},W\right)-\loss\left(Z_{2,i},W\right)\right)\\
&\leq \frac{1}{n}  \sum_{i=1}^{n} \EE \sqrt{2 \dminf{\supersam{2}}{W;U_i} }.
\]
\end{proof}

\begin{proof}[Proof of \cref{thm:gen-kl-m1}]
For any two random measures $P(\supersam{2},U_{J^c},J)$ and $Q(\supersam{2},U)$ on $\parspace$, the Donsker--Varadhan variational formula \citep[Prop.~4.15]{boucheron2013concentration} and the disintegration theorem \citep[Thm.~6.4]{FMP2}, give that
with probability one
\[
\label{eq:dv-kl-kl}
	\KL{Q(\supersam{2},U)}{P(\supersam{2},U_{J^c},J)} = \sup_{g\in \mathcal{G}} \bigg( Q(\supersam{2},U)\left[g\right] - \log P(\supersam{2},U_{J^c},J) \left[\exp g\right]\bigg) 
\]
where $\mathcal{G} = \set{g :P(\supersam{2},U_{J^c},J)(\exp g)<\infty}$. 

Let $g =  \frac{\lambda}{m} \sum_{j\in J} \left(-1\right)^{U_j} \left( \loss\left(Z_{1,j},W\right)-\loss\left(Z_{2,j},W\right)\right) $. First, note that
\*[
\cEE{\supersam{2},U_{J^c},J}{\left[ \frac{\lambda}{m} \sum_{j\in J} \left(-1\right)^{U_j} \left( \loss\left(Z_{1,j},W\right)-\loss\left(Z_{2,j},W\right)\right)\right]}=0.
\]
This is because $\{U_{j}\}_{j\in J}$ are independent of $\supersam{2}$,$U_{J^c},$ and $J$. Moreover, $g$ is $\left[-\lambda,\lambda\right]$-bounded. Therefore, we can use the Hoeffding's lemma to obtain
\*[
\log P(\supersam{2},U_{J^c},J) \left(\exp g\right) \leq \frac{\lambda^2}{2}.
\]
Hence, from \cref{eq:dv-kl-kl}, we conclude that 
\*[
&Q(\supersam{2},U)\left[\frac{1}{m} \sum_{j\in J} \left(-1\right)^{U_j} \left( \loss\left(Z_{1,j},W\right)-\loss\left(Z_{2,j},W\right)\right)\right]\\
&\leq \inf_{\lambda >0} \frac{\KL{Q(\supersam{2},U,J)}{P(\supersam{2},U_{J^c},J)} }{\lambda}  +\frac{\lambda}{2} =  \sqrt{2 \KL{Q(\supersam{2},U)}{P(\supersam{2},U_{J^c},J)}}
\]
almost surely.
Finally, since $J \indep \left(\supersam{2},U\right)$ we get
\*[
&Q(\supersam{2},U)\left[\frac{1}{m} \sum_{j\in J} \left(-1\right)^{U_j} \left( \loss\left(Z_{1,j},W\right)-\loss\left(Z_{2,j},W\right)\right)\right]\\
&= Q(\supersam{2},U) \left[\frac{1}{n} \sum_{i= 1}^{n} \left(-1\right)^{U_j} \left( \loss\left(Z_{1,i},W\right)-\loss\left(Z_{2,i},W\right)\right)\right] \\
&= \EE\sbra{ \Risk{\Dist}{W} - \EmpRisk{S}{W} }
\]
The desired result follows.
\end{proof}

\section{Proofs of \cref{sec:ld-gen-bound}}
\label{apx:proof-ld}
\begin{proof}[Proof of \cref{thm:gen_bound_ld}]
Considering the generalization bound in \cref{thm:gen-kl-m1} and \cref{lemma:variation-mi-kl}, we can write
\begin{align}
\EE\sbra{ \Risk{\Dist}{W} - \EmpRisk{S}{W} }&\leq \EE \sqrt{2\KL{Q_T\left(S\right)}{P_T\left(\tilde{Z}^{2},U_{J^c},J\right)}}\nonumber\\
&\leq \EE \sqrt{  \sum_{t=1}^{T} 2 \cEE{\supersam{2},U,J}\KL{\conditional{Q}}{\conditional{P}}}. \label{eq:kl_varitional}
\end{align}
First, note that from \cref{eq:ldup} it follows that $$\conditional{Q}=\Normal(\mu_{\conditional{Q}}, \frac{2\eta_t}{\beta_t} \id{d}),$$ where the mean is given by 
\*[
&\mu_{\conditional{Q}}=\\
&  \ww_{t-1} - \eta_{t-1} \frac{n-1}{n}\bgrad{S_{J^c}}{\ww_{t-1}}-\frac{\eta_{t-1}}{n}\left(\indic{U_J=1} \ssgrad\left(Z_{1,J},\ww_{t-1}\right) + \indic{U_J=2}  \ssgrad\left(Z_{2,J},\ww_{t-1}\right) \right).
\]
Next, we propose the following construction of $\conditional{P}$. Note that $\conditional{P}$  is $\mathcal{F}_t$-measurable random probability measure where 
$$\mathcal{F}_t=\sigma(S_{J^c},Z_{1,J},Z_{2,J},J,\ww_{0:t-1}).$$ 
Hence we can exploit the information in the trajectory up to time $t$ to construct $\conditional{P}$. In particular, we use the information in $\mathcal{F}_t$ to perform a binary hypothesis testing in which the two hypotheses are defined as 
\begin{align*}
\hypt{1}&: U_J=1, \\
\hypt{2}&: U_J=2.
\end{align*}
Equivalently, $\hypt{1}$ and $\hypt{2}$ can also be described as the hypotheses that $Z_{1,J}$ is a member of the training set and $Z_{2,J}$ is a member of the training set, respectively.  Denote $\pi_t=\left(\pi_{t,1},\pi_{t,2}\right)$  as a probability vector whose $i-$th element shows the belief of the prior at time $t$ that the true hypothesis is $\hypt{i}$ for $i \in \{1,2\}$. Then, we consider the prior as
\[
\conditional{P}=\Normal(\mu_{\conditional{P}}, \frac{2\eta_{t-1}}{\beta_{t-1}} \id{d}),
\]
 where 
\[
\mu_{P_{t|}}=\ww_{t-1} - \eta_{t-1} \frac{n-1}{n}\bgrad{S_{J^c}}{\ww_{t-1}}-\frac{\eta_{t-1}}{n}\left(\pi_{t,1} \ssgrad\left(Z_{1,J},\ww_{t-1}\right) + \pi_{t,2}\ssgrad\left(Z_{2,J},\ww_{t-1}\right) \right).
\]
Here $\pi_1=(\frac{1}{2},\frac{1}{2})$. Then, we construct the the belief vector $\pi_t$ for $t \geq 2$ using the log-likelihood ratio as 
\[
\label{eq:bivec-general-form}
\pi_t = \Big( \theta \big( \log \frac{ \cPr{\mathcal{F}_t}{\hypt{1}}}{ \cPr{\mathcal{F}_t}{\hypt{2}}} \big),1-\theta \big( \log \frac{ \cPr{\mathcal{F}_t}{\hypt{1}}}{ \cPr{\mathcal{F}_t}{\hypt{2}}} \big) \Big),
\]
where $\theta: \Reals \rightarrow [0,1]$. Also, we might expect that the optimal $\theta$ satisfies $\theta(0)=\frac{1}{2}$, $\lim_{x\to \infty} \theta(x) = 1$, and $\lim_{x\to -\infty} \theta(x) = 0$.

Denote probability density function 
$\cPr{\supersam{2},U_{J^c},\hypt{k},\ww_0}{\ww_{1:t-1}}$ as $f_k\left(\ww_{1:t-1}\right)$ for $k \in \{1,2\}$. Due to Markov structure of the update rule in \cref{eq:ldup}, we have 
\begin{align}
&f_k\left(\ww_{1:t-1}\right)= \nonumber\\
&  \prod_{i=1}^{t-1} \left(\frac{\beta_{i-1}}{4 \pi \eta_{i-1}}\right)^{\frac{d}{2}} \exp\left(-\frac{ \beta_{i-1} \| \ww_{i}-\ww_{i-1}+\eta_{i-1}\frac{n-1}{n}\bgrad{S_{J^c}}{\ww_{i-1}}+\frac{\eta_{i-1}}{n}\ssgrad\left(Z_{k,J},\ww_{i-1}\right)  \|^2 }{4\eta_{i-1}}\right).
\label{eq:markov_prop_ld}
\end{align}
Here, \cref{eq:markov_prop_ld} is obtained by the Markov property of the update rule in \cref{eq:ldup}. Then, since the prior distribution on $\hypt{1}$ and $\hypt{2}$ is uniform,  we have

\begin{align}
\log\frac{ \cPr{\mathcal{F}_t}{\hypt{1}}}{ \cPr{\mathcal{F}_t}{\hypt{2}}} &= \log\frac{f_{1}(\ww_{1:t-1})}{f_2(\ww_{1:t-1})}\\
&=Y_{t,2}-Y_{t,1},
\end{align}
where $Y_{t,1}$ and $Y_{t,2}$ are given by
\begin{equation}
\begin{aligned}
Y_{t,1} &\triangleq \sum _{i=1}^{t-1} \frac{\beta_{i-1}}{4\eta_{i-1}} \| \ww_{i}-\ww_{i-1}+\eta_{i-1}\frac{n-1}{n}\bgrad{S_{J^c}}{\ww_{i-1}}+\frac{\eta_{i-1}}{n}\ssgrad\left(Z_{1,J},\ww_{i-1}\right)  \|^2, \\
Y_{t,2} &\triangleq \sum _{i=1}^{t-1} \frac{\beta_{i-1}}{4\eta_{i-1}}  \| \ww_{i}-\ww_{i-1}+\eta_{i-1}\frac{n-1}{n}\bgrad{S_{J^c}}{\ww_{i-1}}+\frac{\eta_{i-1}}{n}\ssgrad\left(Z_{2,J},\ww_{i-1}\right)  \|^2.
\label{eq:y1-y2}
\end{aligned}
\end{equation}
Therefore, the belief vector is given by 
\[
\pi_t = \Big(\theta\big(  (Y_{t,2}-Y_{t,1})\big), 1-\theta\big( (Y_{t,2}-Y_{t,1})\big)\Big),
\]
where $Y_{0,1}=Y_{0,2}=0$ and for $t\geq 2$, $Y_{t,1}$ and $Y_{t,2}$ are given by \cref{eq:y1-y2}. To conclude the proof, we obtain
\begin{align}
& \KL{Q_T\left(S\right)}{P_T\big(\supersam{2},U_{J^c},J\big)} \leq \sum_{t=1}^{T} \cEE{\supersam{2},U,J}{\KL{Q_{t\vert}}{P_{t\vert}}}\\
&=\sum_{t=1}^{T} \cEE{\supersam{2},U,J}{\frac{ \beta_{t-1} \eta_{t-1} \| (\indic{U_J=1}-\pi_{t,1}) \ssgrad\left(Z_{1,J},W_{t-1}\right)+(\indic{U_J=2}-\pi_{t,2}) \ssgrad\left(Z_{2,J},W_{t-1}\right) \|^2}{4n^2}}\\
&=\sum_{t=1}^{T} \cEE{\supersam{2},U,J}{ \frac{ \beta_{t-1} \eta_{t-1} (\indic{U_J=1}-\pi_{t,1})^2 \|  \ssgrad\left(Z_{1,J},W_{t-1}\right)-\ssgrad\left(Z_{2,J},W_{t-1}\right) \|^2}{4n^2}}
\label{eq:kl_upper}
\end{align}
Finally, plugging \cref{eq:kl_upper} into \cref{eq:kl_varitional}, we get the desired result in \cref{eq:gen_ld_final}.
\end{proof}

\section{Conditional Han's Inequality}

\label{app:han-lemma}
\begin{lemma}
\label{lem:han-lemma}
Let $(X_1,\hdots,X_n,Y)$ be $n+1$-dimensional random variable where $X_1,\hdots,X_N$ are discrete random variables. Then, 
\*[
\frac{1}{k {n \choose k}} \sum_{T \in \range{n}_k} \entr{X_T \vert Y}
\]
 is decreasing in k.
\end{lemma}
\begin{proof}
For notational convenience, let $\overline{H}_k(X_{\range{n}}\vert Y)=\frac{1}{ {n \choose k}} \sum_{T \in \range{n}_k} \entr{X_T \vert Y}$. Note that if we manage to show that 
\[
\label{eq:main-step-han}
\overline{H}_k(X_{\range{n}}\vert Y)-\overline{H}_{k-1}(X_{\range{n}}\vert Y)\leq \overline{H}_{k+1}(X_{\range{n}}\vert Y)-\overline{H}_{k}(X_{\range{n}}\vert Y),
\]
then the result in \cref{lem:han-lemma} follows. To show \cref{eq:main-step-han}, we can write
\begin{align}
&\entr{X_1,\hdots,X_{k+1}\vert Y} + \entr{X_1,\hdots,X_{k-1}\vert Y} \nonumber\\
=& \entr{X_1,\hdots,X_k\vert Y} + \entr{X_{k+1}\vert X_1,\hdots,X_k, Y} +  \entr{X_1,\hdots,X_{k-1}\vert Y} \\ 
\leq & \entr{X_1,\hdots,X_k\vert Y} + \entr{ X_{k+1} \vert X_1,\hdots,X_{k-1},Y} +\entr{X_1,\hdots,X_{k-1}\vert Y} \label{eq:han-cond} \\
=&\entr{X_1,\hdots,X_k\vert Y} + \entr{X_1,\hdots,X_{k-1},X_{k+1}\vert Y}.
\end{align}
Here in \cref{eq:han-cond}, we drop $X_k$ from the condition in the second term. Therefore, we have
\[
\label{eq:han-before-end}
&\entr{X_1,\hdots,X_{k+1}\vert Y} + \entr{X_1,\hdots,X_{k-1}\vert Y} \leq \entr{X_1,\hdots,X_k\vert Y} + \entr{X_1,\hdots,X_{k-1},X_{k+1}\vert Y}.
\] 
Then, by averaging \cref{eq:han-before-end} over all $n!$ permutation of $\{1,\hdots,n\}$ , we get the desired result in \cref{eq:main-step-han}.
\end{proof}
\section{Details of Experiments}
\label{app:experminet-details}
In this section, we discuss the details behind the experiments as well as the details of minimizing the generalization bound in \cref{thm:gen_bound_ld}.
\subsection{Network architectures and learning curve}
\cref{tbl:hyp-mnist-mlp,tbl:hyp-mnist-cnn,tbl:hyp-fmnist,tbl:hyp-cnn} summarize the hyper-parameters we used for the experiments. Also, in \cref{fig:learning-curves} we plot the learning curves for the experiments reported in \cref{subsec:empirical-results}.

\begin{table}[!h]
\small
\centering
\input{./tables/table-mnist-mlp-hyperparam.tex}
\caption{Details of Experiments reported for MNIST with MLP}
\label{tbl:hyp-mnist-mlp}
\end{table}

\begin{table}[!h]
\small
\centering
\input{./tables/table-mnist-cnn-hyperparam.tex}

\caption{Details of Experiments reported for MNIST with CNN}
\label{tbl:hyp-mnist-cnn}
\end{table}

\begin{table}[!h]
\small
\centering
\input{./tables/table-Fashion-MNIST-hyperparam.tex}
\caption{Details of Experiments reported for Fashion-MNIST with CNN}
\label{tbl:hyp-fmnist}
\end{table}

\begin{table}[!h]
\small
\centering
\input{./tables/table-CIFAR-hyperparam.tex}

\caption{Details of Experiments reported for CIFAR10 with CNN}
\label{tbl:hyp-cnn}
\end{table}

\begin{figure}
\centering
\begin{subfigure}[b]{.48\textwidth}
  \centering
  \includegraphics[width=1\linewidth, height=7cm]{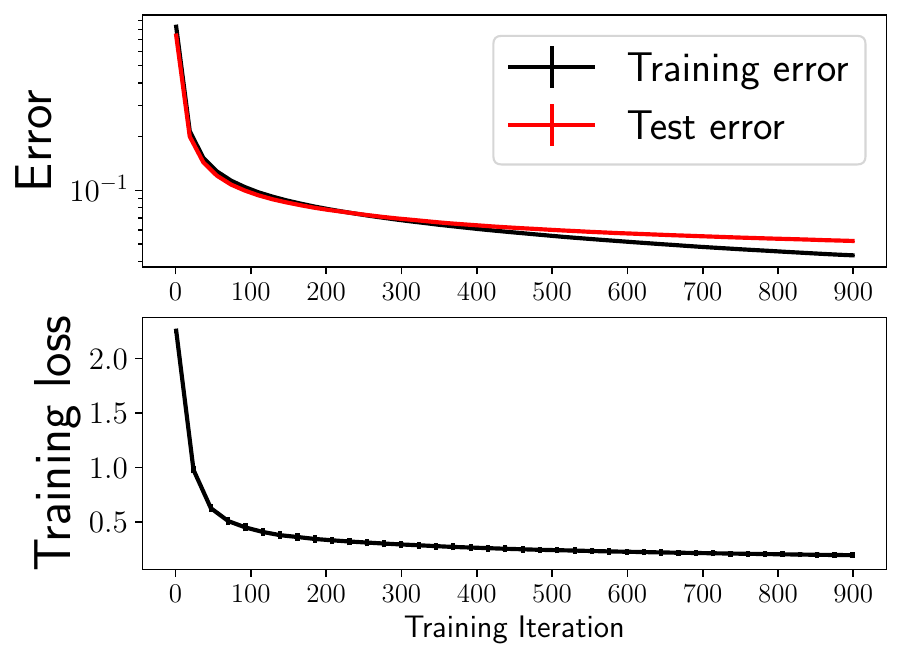}
  \caption{MNIST with MLP}
\end{subfigure}
\begin{subfigure}[b]{.48\textwidth}
  \centering
  \includegraphics[width=1\linewidth, height=7cm]{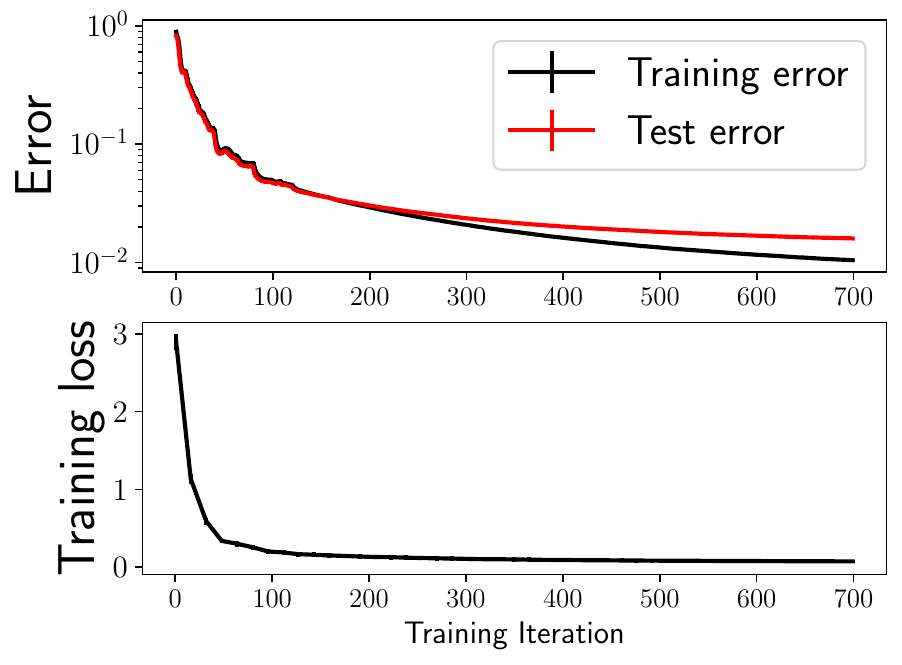}
	\caption{MNIST with CNN}
\end{subfigure}
\newline
\begin{subfigure}[b]{.48\textwidth}
  \centering
  \includegraphics[width=1\linewidth, height=7cm]{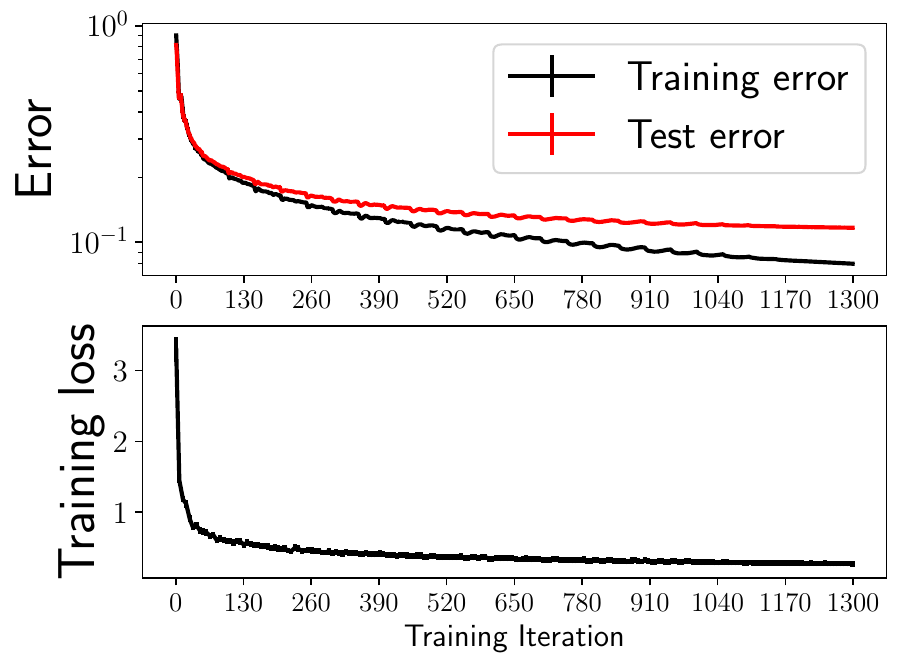}
	\caption{Fashion-MNIST}
\end{subfigure}
\begin{subfigure}[b]{.48\textwidth}
  \centering
  \includegraphics[width=1\linewidth, height=7cm]{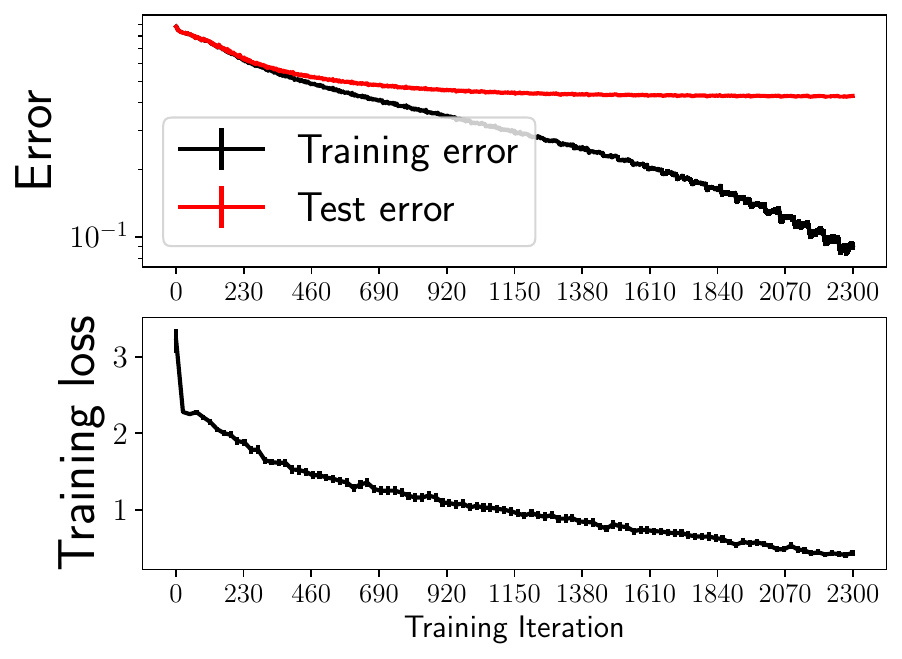}
	\caption{CIFAR10}
\end{subfigure}%
\caption{Learning curves. These plots show the training error, error on the test set, and the training loss. The loss functions is cross-entropy. Note y-axes for the error plots are log-scale. }
\label{fig:learning-curves}
\end{figure}

\subsection{Optimizing the bound over the choice of $\theta$ function}
Our generalization bound in \cref{thm:gen_bound_ld} consists of an infimum  over the functions in $\Theta$. To study the impact of infimum, we consider the family of functions $\Theta$ given by
\*[
\Theta = \{\theta_{a}(x)| \exists  a>0 \ \text{such that} \ \theta_{a}(x)=\frac{1}{2}(1+\mathrm{erf}(\frac{x}{a})) \ \text{or} \  \theta_{a}(x)=\frac{1}{2}(1+\mathrm{tanh}(\frac{x}{a})) \}.
\]
Then, we divide the samples of the optimization trajectory into two sets of equal size: training set and the test set. Then, we optimize over $a$ to find the $\theta_{a^{\star}}(x)$ that achieves the minimum expected generalization over the training set. The numbers reported in \cref{tbl:sum-results} are based on the evaluation of $\theta_{a^{\star}}(x)$ over the test set. Thus, the number reported in \cref{tbl:sum-results} are unbiased estimate of the generalization bound in \cref{thm:gen_bound_ld}.

%% file: tables/table-mnist-mlp-hyperparam.tex
\begin{tabular}{|c|c|}
\hline
Dataset & MNIST \\ \hline
Architecture & MLP(784-500-500-10) \\ \hline
$\eta_t$       & $0.06 \times (0.95)^{\lceil \frac{t}{50} \rceil}$      \\ \hline
$\frac{2\eta_t}{\beta_t}$     &    $10^{-8} + (3\times 10^{-6}-10^{-8})\times \exp(-0.5 \lceil \frac{t}{50} \rceil)$      \\ \hline
Number of iterations     & $900$       \\ \hline
Final training error & $4.33 \pm 0.01\% $     \\ \hline
Generalization error     &  $0.88 \pm 0.01 \%$        \\ \hline
Number of training examples   &  $20000$   \\ \hline
Number of runs     & $100$      \\ \hline
\end{tabular}

%% file: tables/table-mnist-cnn-hyperparam.tex
\begin{tabular}{|c|c|}
\hline
Dataset & MNIST \\ \hline
Architecture &  \shortstack{ CL($5\times5 (32)$)-MaxPool($2\times 2)$-CL($5\times5 (64)$)\\ MaxPool($2\times 2)$-FC(128)-FC(10)}\\ \hline
$\eta_t$       & $0.05 \times (0.90)^{\lceil \frac{t}{40} \rceil}$      \\ \hline
$\frac{2\eta_t}{\beta_t}$     &    $10^{-8} + ( 10^{-5}-10^{-8})\times \exp(-0.5 \lceil \frac{t}{40} \rceil)$      \\ \hline
Number of iterations     & $700$       \\ \hline
Final training error & $2.59 \pm 0.01\%$     \\ \hline
Generalization error     &  $0.55 \pm 0.01\% $       \\ \hline
Number of training examples   &  $20000$   \\ \hline
Number of runs     & $100$      \\ \hline
\end{tabular}

%% file: tables/table-Fashion-MNIST-hyperparam.tex
\begin{tabular}{|c|c|}
\hline
Dataset & Fashion-FMNIST \\ \hline
Architecture & \shortstack{CL($5\times5 (32)$)-MaxPool($2\times 2)$-CL($5\times5 (64)$)\\MaxPool($2\times 2)$-FC(200)-FC(10)}  \\ \hline
$\eta_t$       & $0.07 \times (0.95)^{\lceil \frac{t}{50} \rceil}$      \\ \hline
$\frac{2\eta_t}{\beta_t}$     &    $5\times 10^{-8} + ( 7\times 10^{-6}-5\times 10^{-8})\times \exp(-0.3 \lceil \frac{t}{50} \rceil)$      \\ \hline
Number of iterations     & $1300$       \\ \hline
Final training error & $7.96 \pm 0.03\%$    \\ \hline
Generalization error     &   $3.71 \pm 0.03\%$       \\ \hline
Number of training examples   &  $20000$   \\ \hline
Number of runs     & $100$      \\ \hline
\end{tabular}

%% file: tables/table-CIFAR-hyperparam.tex
\begin{tabular}{|c|c|}
\hline
Dataset & CIFAR10 \\ \hline
Architecture &     \shortstack{CL($3\times3 (32)$)-MaxPool($2\times 2)$-CL($3\times3 (64)$) \\ CL($3\times3 (32)$)-MaxPool($2\times 2)$-FC(128)-FC(10)}  \\ \hline
$\eta_t$       & $0.15 \times (0.98)^{\lceil \frac{t}{50} \rceil}$      \\ \hline
$\frac{2\eta_t}{\beta_t}$     &    $10^{-9} + ( 3\times 10^{-5}-10^{-9})\times \exp(-0.3 \lceil \frac{t}{50} \rceil)$      \\ \hline
Number of iterations     & $2300$       \\ \hline
Final training error & $9.39 \pm 0.46\%$    \\ \hline
Generalization error     &   $32.89 \pm 0.44\%$       \\ \hline
Number of training examples   &  $15000$   \\ \hline
Number of runs     & $100$      \\ \hline
\end{tabular}